\documentclass[a4paper]{article} 
\usepackage[margin=1in]{geometry}

\usepackage{enumitem}

\usepackage{amsmath,amssymb,amsthm}
\usepackage{bm}
\usepackage{natbib}
\usepackage{graphicx,here,comment}
\usepackage{algorithmic}
\usepackage{algorithm}
\usepackage{xcolor}
\usepackage{wrapfig}
\usepackage{footnote}
\makesavenoteenv{table}  

\usepackage{mathtools}

\usepackage{hyperref}

\newtheorem{theorem}{Theorem}

\newtheorem{lemma}[theorem]{Lemma}

\newtheorem{assumption}{Assumption}
\theoremstyle{definition}

\newtheorem{remark}{Remark}

\newcommand{\rbr}[1]{\left(#1\right)}
\newcommand{\sbr}[1]{\left[#1\right]}

\newcommand{\R}{\mathbb{R}}
\newcommand{\N}{\mathbb{N}}
\newcommand{\mF}{\mathcal{F}}

\newcommand{\mB}{\mathcal{B}}
\newcommand{\mC}{\mathcal{C}}

\newcommand{\mH}{\mathcal{H}}

\newcommand{\mS}{\mathcal{S}}

\newcommand{\mO}{\mathcal{O}}

\newcommand{\mN}{\mathcal{N}}
\newcommand{\mX}{\mathcal{X}}

\newcommand{\Ep}{\mathbb{E}}
\renewcommand{\Pr}{\mathbb{P}}

\renewcommand{\hat}{\widehat}
\renewcommand{\tilde}{\widetilde}

\newcommand{\argmax}{\operatornamewithlimits{argmax}}

\def\bX{\mathbf{X}}

\def\bK{\mathbf{K}}

\def\ba{\bm{a}}

\def\bk{\bm{k}}

\def\bI{\bm{I}}

\def\bx{\bm{x}}
\def\by{\bm{y}}

\usepackage{color}

\usepackage{newtxtext,newtxmath}

\newcommand{\secref}[1]{Section~\ref{#1}}
\newcommand{\appref}[1]{Appendix~\ref{#1}}
\newcommand{\algoref}[1]{Algorithm~\ref{#1}}
\newcommand{\asmpref}[1]{Assumption~\ref{#1}}
\newcommand{\thmref}[1]{Theorem~\ref{#1}}
\newcommand{\lemref}[1]{Lemma~\ref{#1}}

\newcommand{\figref}[1]{Figure~\ref{#1}}
\newcommand{\tabref}[1]{Table~\ref{#1}}

\usepackage{authblk}

\title{Near-Optimal Algorithm for Non-Stationary Kernelized Bandits}
\date{}
\author[1]{Shogo Iwazaki}
\author[2,3]{Shion Takeno}

\affil[1]{MI-6 Ltd.}
\affil[2]{Nagoya University}
\affil[3]{RIKEN AIP}

\affil[1]{\texttt{{shogo.iwazaki@gmail.com}}}
\affil[2]{\texttt{{takeno.shion.m6@f.mail.nagoya-u.ac.jp}}}

\begin{document}
\maketitle

\begin{abstract}
This paper studies a non-stationary kernelized bandit (KB) problem, also called time-varying Bayesian optimization, where one seeks to minimize the regret under an unknown reward function that varies over time.
In particular, we focus on a near-optimal algorithm whose regret upper bound matches the regret lower bound.
For this goal, we show the first algorithm-independent regret lower bound for non-stationary KB with squared exponential and Mat\'ern kernels, which reveals that an existing optimization-based KB algorithm with slight modification is near-optimal.
However, this existing algorithm suffers from feasibility issues due to its huge computational cost.
Therefore, we propose a novel near-optimal algorithm called restarting phased elimination with random permutation (R-PERP), which bypasses the huge computational cost.
A technical key point is the simple permutation procedures of query candidates, which enable us to derive a novel tighter confidence bound tailored to the non-stationary problems.
\end{abstract}

\section{Introduction}
\label{sec:intro}

Kernelized bandit (KB) problem~\citep{srinivas10gaussian}, also called Gaussian process bandit or Bayesian optimization, is one of the important sequential decision-making problems where one seeks to minimize the regret under an unknown reward function via sequentially acquiring function evaluations.
As the name suggests, in the KB problem, the underlying reward function is assumed to be an element of reproducing kernel Hilbert space (RKHS) induced by a known fixed kernel function.
KB has been applied in many applications, such as materials discovery~\citep{ueno2016combo}, drug discovery~\citep{korovina2020-chemBO}, and robotics~\citep{berkenkamp2023bayesian}.
In addition, the near-optimal KB algorithms, whose regret upper bound matches the regret lower bound derived in \citet{scarlett2017lower}, have been shown~\citep{camilleri2021high,salgia2021domain,li2022gaussian,salgiarandom}.

Non-stationary KB~\citep{bogunovic2016time} considers the optimization under a non-stationary environment; that is, the reward function may change over time within some RKHS.
This modification is crucial in many practical applications where an objective function varies over time, such as financial markets~\citep{heaton1999stock} and recommender systems~\citep{Hariri2015adapting}.
For example, \cite{zhou2021no,deng2022weighted} have proposed upper confidence bound (UCB)-based algorithms for the non-stationary KB problem and derived the upper bound of the cumulative regret.
Recently, \citet{hong2023optimization} have proposed an optimization-based KB (OPKB\footnote{\citet{hong2023optimization} originally uses the terminology OPKB only for the algorithm under the stationary KB problems. For simplicity, in this paper, we use the terminology OPKB for the general KB algorithms constructed on optimization-based procedures.}) algorithm, which achieves a tighter regret upper bound than that of \citet{zhou2021no,deng2022weighted}.
This result implies that the OPKB algorithm is near-optimal for a linear kernel.

However, there are still two open problems with the non-stationary KB problem.
First, although the OPKB algorithm is near-optimal for the linear kernel, the optimality for squared exponential (SE) and the Mat\'ern kernels has not been revealed.
Since SE and Mat\'ern kernels are widely used in practice and of interest in theory~\citep{shahriari2016taking}, revealing the optimality for these two kernels is valuable.
Second, although the OPKB algorithm achieves the known-best regret upper bound, the OPKB suffers from feasibility issues when the cardinality of the input set is huge.
The OPKB algorithm requires running two costly procedures: (i) the construction of an explicit feature map of the kernel and (ii) the optimization whose dimension is $|\mX|$, where $\mX$ is the input set of the reward function.
%
In particular, the procedure (ii) requires $\mO(|\mX|^3)$ computation for every steps in the $|\mX|$-dimensional optimization.
Therefore, when the input set $\mX$ is huge, even running the OPKB algorithm can be unrealistic.

\begin{table*}[t]
    \centering
    \caption{The comparison of existing and our algorithms for regrets and computational costs under a finite input set $\mX \subset \R^d$. We denote $V_T$ as the upper bound of the total variation of the sequence of underlying reward functions (precise definition is in \asmpref{asmp:nons}) and $M$ as the total iteration to solve $|\mX|$-dimensional optimization problem in the OPKB. 
    For the regret upper bound described in the table, we assume that $V_T$ satisfies $V_T > c$, where $c > 0$ is any fixed constant.
    }
    \label{tab:comp_alg}
    \begin{tabular}{c|c|c|c}
        Algorithm & Regret (SE) & Regret (Mat\'ern) &
             Computational cost 
             at step $t \leq T$ \footnotemark \footnotemark 
         \\ \hline
        R/SW-GP-UCB & $\tilde{\mO}(T^{\frac{3}{4}} V_T^{\frac{1}{4}})$ & $\tilde{\mO}(T^{\frac{12\nu + 13d}{16\nu + 8d}} V_T^{\frac{1}{4}}) $ & $\mO(|\mX|t^2)$ \\
        WGP-UCB & $\tilde{\mO}(T^{\frac{3}{4}} V_T^{\frac{1}{4}})$ & 
        $\tilde{\mO}(T^{\frac{12\nu + 13d}{16\nu + 8d}} V_T^{\frac{1}{4}}) $ & $\mO(|\mX|t^2)$ \\
        OPKB & $\tilde{\mO}(T^{\frac{2}{3}} V_T^{\frac{1}{3}})$ 
        & $\tilde{\mO}(T^{\frac{4\nu + 3d}{6\nu + 3d}} V_T^{\frac{1}{3}})$ & $\mO(M|\mX|^3)$ \\
        R-PREP (Ours) & $\tilde{\mO}(T^{\frac{2}{3}} V_T^{\frac{1}{3}})$ & 
$\tilde{\mO}(T^{\frac{2\nu + d}{3\nu + d}} V_T^{\frac{\nu}{3\nu + d}})$ & $\mO(|\mX|t^2)$ \\ \hline
        Lower bounds (Ours) & $\tilde{\Omega}(T^{\frac{2}{3}} V_T^{\frac{1}{3}})$ & $\Omega(T^{\frac{2\nu + d}{3\nu + d}} V_T^{\frac{\nu}{3\nu + d}})$
        &
    \end{tabular}
\end{table*}
\footnotetext[2]{Strictly speaking, the dependence of $t$ can become smaller in R/SW-GP-UCB and our R-PERP algorithms since these algorithms discard past training data of surrogate models at some intervals. Furthermore, OPKB and R-PERP choose query points using batch-based calculations; therefore, several time steps are skipped, and the computational costs of such time steps are zero.}
\footnotetext[3]{Note that KB problem usually focuses on the regime of $T \ll |\mX|$.}

This paper tackles the above two open problems.
Our contributions are summarized as follows:
\begin{itemize}
    \item We show the first algorithm-independent lower bounds for the non-stationary KB problem for SE and Mat\'ern kernel. Our results shows that any algorithms suffer $\tilde{\Omega}(V_T^{\frac{1}{3}} T^{\frac{2}{3}})$ and $\tilde{\Omega}(V_T^{\frac{\nu}{3\nu + d}} T^{\frac{2\nu + d}{3\nu + d}})$ worst-case regret for SE and Mat\'ern kernel, respectively. Here, $V_T$ denotes the upper bound of the total variation of the sequence of underlying reward 
functions. (See \asmpref{asmp:nons} for the formal definition.)
    \item Based on our lower bound, we confirm that the existing OPKB algorithm achieves nearly optimal regret for the SE kernel. In addition, we show that the OPKB algorithm with a slight modification achieves near-optimal regret for the Mat\'ern kernels under known $V_T$.
    \item We further propose a novel phased elimination (PE)~\citep{li2022gaussian} based algorithm, which bypasses the computationally hard procedures of OPKB. Our proposed algorithm, restarting PE with random permutation (R-PERP), combines the existing PE-based algorithm with simple permutation procedures of query candidates. 
    The key to our regret analysis is the derivation of the tighter confidence bound (CB) tailored to the non-stationary setting based on such a permutation procedure.
\end{itemize}
In \tabref{tab:comp_alg}, we summarized the computational cost and the regret of the existing and our algorithms for SE and Mat\'ern kernels.

\paragraph{Limitations}
The main limitation of this paper is that our near-optimal guarantees of OPKB and R-PERP rely on the prior knowledge of the upper bound $V_T$. However, it is worth noting that, even if $V_T$ is unknown, our R-PERP algorithm achieves tighter regret than existing non-stationary KB algorithms \citep{zhou2021no,deng2022weighted} except for OPKB (see, \appref{sec:unknown_vt}). Specifically, except for OPKB, R-PERP is the only algorithm whose regret upper bound is always sub-linear in Mat\'ern kernels for fixed $V_T$.
Furthermore, in contrast to our proposed R-PERP algorithm, OPKB is hard to apply the problem whose input set $\mathcal{X}$ is huge; therefore, the proposal of R-PERP is also valuable for situations whose $V_T$ is unknown.

\subsection{Related Works}
The KB problems under stationary environments are 
extensively studied, and several algorithms are 
proposed, including Gaussian process UCB (GP-UCB) and Thompson sampling (TS)-based algorithm~\citep{srinivas10gaussian,chowdhury2017kernelized}.
Furthermore, \citet{scarlett2017lower} derive the 
regret lower bound of stationary KB for SE and Mat\'ern kernels and 
shows that existing regret upper bounds of UCB or TS-based algorithms are strictly sub-optimal for Mat\'ern kernel. Recently, several works have tackled constructing an algorithm whose regret nearly matches lower bounds~\citep{camilleri2021high,salgia2021domain,li2022gaussian,salgiarandom}.

As for the non-stationary KB problem, \citet{bogunovic2016time} first propose the UCB-based algorithms, called resetting GP-UCB (R-GP-UCB) and time-varying GP-UCB (TV-GP-UCB), which is based on the restart and reset strategy and the smoothly forgetting strategy for past observations, respectively. 
The analysis of \citet{bogunovic2016time} is based on Bayesian assumption whose reward functions follow some Gaussian process. 
The frequentist analysis of R-GP-UCB was later shown in \citet{zhou2021no}. 
\citet{zhou2021no} further proposed sliding window GP-UCB (SW-GP-UCB), which uses the training data in sliding window.
\citet{deng2022weighted} have proposed 
another type of UCB-based algorithm called weighted GP-UCB (WGP-UCB), which is based on the modified version of the GP whose past observed outputs are weighted in a time-dependent manner.
The regret of these UCB-based algorithms has been shown to become $\mathcal{O}(\gamma_T^{\frac{7}{8}} T^{\frac{3}{4}} V_T^{\frac{1}{4}})$, where $\gamma_T$ is the maximum information gain, which represents the complexity of the problem (precise definition is described in \secref{sec:prelim}). 
Recently, the OPKB algorithm proposed by \citet{hong2023optimization} has been shown to achieve $\tilde{\mO}(\gamma_T^{\frac{1}{3}} T^{\frac{2}{3}} V_T^{\frac{1}{3}})$ regret, which nearly matches the lower bound for the linear kernel~\citep{cheung2019learning}. 
As mentioned above, the computational cost of OPKB may 
become huge when the learner uses SE or Mat\'ern kernels. UCB-based algorithms do not have such computational issues, while the regret guarantees are worse than that of OPKB. 

\section{Preliminaries}
\label{sec:prelim}
\paragraph{Problem Setting}
We consider the reward maximization problem under a non-stationary environment. Let $f_t: \mX \rightarrow \R$ 
be an unknown reward function at step $t$. The input domain 
$\mX$ is a compact subset of $\R^d$. The learner 
sequentially chooses the input $\bx_t$ at step $t$; then, 
the environment reveals the noisy observation $y_t \coloneqq f_t(\bx_t) + \epsilon_t$, where $\epsilon_t$ is a zero-mean noise random variable.
In this paper, we assume the functions $(f_t)_{t \in {\N_+}}$ 
are determined by the environment \emph{obliviously}; namely, 
all functions $(f_t)_{t \in {\N_+}}$ are fixed by the environment before the learner chooses the first input $\bx_1$. In the aforementioned setup, the learner's goal is to minimize the following cumulative regret $R_T$:
\begin{equation}
    R_T = \sum_{t \in [T]} f_t(\bx_t^\ast) - f_t(\bx_t),
\end{equation}
where $\bx_t^\ast \in \argmax_{\bx \in \mathcal{X}}f_t(\bx)$ 
and $[T] = \{1, \ldots, T\}$.

For our theory, we make the following assumptions.
\begin{assumption}[Assumption for noise]
    \label{asmp:noise}
    The noise sequence $(\epsilon_t)_{t \in \N_+}$ is mutually independent. Furthermore, the noise $\epsilon_t$ is $\rho$-sub-Gaussian random variable for any $t \in \N_+$; namely, $\Ep[\exp(\eta \epsilon_t)] \leq \exp(\eta^2 \rho^2/2)$
    for all $\eta \in \R$.
\end{assumption}

\begin{assumption}[Assumption for reward functions]
    \label{asmp:func}
    Each function $f_t$ is an element of known RKHS with the bounded Hilbert norm.
    Let $k: \mX \times \mX \rightarrow \R$ 
    and $\mH_k$ be a known positive definite kernel 
    and its corresponding RKHS. Then, we assume that $f_t \in \mH_k$
    and $\|f_t\|_{\mH_k} \leq B < \infty$ hold for any $t \in \N_+$,
    where $\|\cdot\|_{\mH_k}$ is the Hilbert norm on $\mH_k$.
    Furthermore, suppose that $k(\bx, \bx) \leq 1$ holds for all $\bx \in \mX$.
\end{assumption}

\begin{assumption}[Assumption for non-stationarity]
    \label{asmp:nons}
    The total drift $\sum_{t=2}^T \|f_t - f_{t-1}\|_{\infty}$ of functions $(f_t)_{t \in [T]}$ is bounded as $\sum_{t=2}^T \|f_t - f_{t-1}\|_{\infty} \leq V_T$, where $\|f_t - f_{t-1}\|_{\infty} = \sup_{\bx \in \mX}|f_t(\bx) - f_{t-1}(\bx)|$.
\end{assumption}

Assumptions~\ref{asmp:noise} and \ref{asmp:func}
are standard for ordinary KB literature~\citep{srinivas10gaussian,chowdhury2017kernelized,vakili2021optimal}. 
Specifically, as for \asmpref{asmp:func}, we focus on the following 
SE kernel $k_{\text{SE}}$ and Mat\'ern kernel 
$k_{\text{Mat\'ern}}$:
\begin{align}
    k_{\text{SE}}(\bx, \tilde{\bx}) &= \exp \rbr{-\frac{\|\bx - \tilde{\bx}\|_2^2}{2 \ell^2}}, \\
    k_{\text{Mat\'ern}}(\bx, \tilde{\bx}) &= \frac{2^{1-\nu}}{\Gamma(\nu)} \rbr{\frac{\sqrt{2\nu} \|\bx - \tilde{\bx}\|_2}{\ell}} J_{\nu}\rbr{\frac{\sqrt{2\nu} \|\bx - \tilde{\bx}\|_2}{\ell}},
\end{align}
where $\ell > 0$ and $\nu > 0$ denotes the length-scale and smoothness parameter, respectively. Furthermore, $\Gamma(\cdot)$ and $J_\nu(\cdot)$ 
are the Gamma and modified Bessel functions, respectively.
\asmpref{asmp:nons} is also standard for non-stationary KB~\citep{zhou2021no,deng2022weighted,hong2023optimization}, and the increasing speed of the upper bound of the total variation $V_T$ has an important role in characterizing both regret lower and upper bounds.

\paragraph{Gaussian Process Model}
Gaussian process (GP) model~\citep{Rasmussen2005-Gaussian} is a useful tool for estimating 
the underlying function while quantifying the uncertainty of its estimate; thus, GP is commonly used to construct an algorithm for the KB problem. Here, let us consider the Bayesian modeling 
of $f$ whose prior is $\mathcal{GP}(0, k)$, where 
$\mathcal{GP}(0, k)$ denotes the mean-zero GP with covariance 
function $k$. 
Given the input data $\bX_t \coloneqq (\bx_1, \ldots, \bx_t)^{\top}$ and the corresponding outputs $\by_t \coloneqq (y_1, \ldots, y_t)^{\top}$, 
the posterior is again the GP, whose posterior mean $\mu(\bx; \bX_t, \by_t)$ and variance $\sigma^2(\bx; \bX_t)$ of $f(\bx)$ are defined as follows:
\begin{align}
    \mu(\bx; \bX_t, \by_t) &= \bk(\bx, \bX_t)^{\top} (\bK(\bX_t, \bX_t) + \lambda \bI_t)^{-1} \by_t, \\
    \begin{split}
    \sigma^2(\bx; \bX_t) &= k(\bx, \bx) 
    - \bk(\bx, \bX_t)^{\top} (\bK(\bX_t, \bX_t) + \lambda \bI_t)^{-1} \bk(\bx, \bX_t), 
    \end{split}
\end{align}
where $\lambda > 0$ is a noise variance parameter, and 
$\bk(\bx, \bX_t) \coloneqq [k(\bx, \bx_i)]_{i \in [t]} \in \R^t$ and $\bK(\bX_t, \bX_t) \coloneqq [k(\bx_i, \bx_j)]_{i,j \in [t]} \in \R^{t \times t}$ are the kernel vector and matrix, respectively. Furthermore, $\bI_t \in \R^{t \times t}$ denotes the identity matrix.

Finally, we define the kernel-dependent quantity $\gamma_T$ as 
$\gamma_T = \frac{1}{2} \sup_{\bX_t} \ln{\det{(\bI_t + \lambda^{-1} \bK(\bX_t, \bX_t))}}$. If we suppose that $f$ follows GP, 
$\gamma_T$ is equal to the maximum amount of 
the information gain of $f$ up to $T$ observations~\citep{srinivas10gaussian}; thus, the quantity $\gamma_T$ is called \emph{maximum information gain} (MIG). MIG characterizes the complexity of the KB problem, and its upper bound increases sub-linearly in several commonly used kernels. For example, $\gamma_T = O(\ln^{d+1} T)$ in SE kernel, and $\gamma_T = O(T^{\frac{d}{2\nu + d}} \ln^{\frac{2\nu}{2\nu + d}} T)$
in Mat\'ern kernel with $\nu > 1/2$~\citep{vakili2021information}.

\paragraph{Phased Elimination}
PE~\citep{li2022gaussian,bogunovic2022robust}, referred to as batched pure exploration in \citet{li2022gaussian}, is a near-optimal KB algorithm in the stationary environment for the SE and Mat\'ern kernels.
PE divides the entire time horizons $T$ 
into some \emph{batches} whose sizes are exponentially increasing.
In each batch, PE performs the (non-adaptive) maximum variance reduction algorithm~\citep{vakili2021optimal} in \emph{potential maximizers}, which is a candidate input whose UCB exceeds the maximum of lower CB.
By the above procedures, PE achieves the regret upper bound $\tilde{\mO}(\sqrt{T \gamma_T})$ while the regret upper bound of classical GP-UCB and TS is $\tilde{\mO}(\gamma_T\sqrt{T})$.
We combine this PE algorithm, the restart and reset strategy, and the random permutation of the query candidates for the non-stationary KB problem.

\section{Lower Bound for Non-Stationary Kernelized Bandits}
\label{sec:lower_bound}
Our first main result is the following \thmref{thm:lb}, which 
shows the algorithm-independent lower bound of the non-stationary KB problem 
with SE or Mat\'ern kernel.

\begin{theorem}[Lower bound]
\label{thm:lb}
Fix any $T \in \mathbb{N}_+$, $V_T > 0$, $B > 0$, and $\rho > 0$. Furthermore, assume $\mX = [0, 1]^d$, and 
$(\epsilon_t)_{t \in \mathbb{N}_+}$ is the noise sequence of independent Gaussian random variables $\epsilon_t \sim \mN(0, \rho^2)$ for all $t \in \mathbb{N}_+$.
\begin{itemize}
    \item Suppose $k = k_{\mathrm{SE}}$, $\tilde{C}_{\mathrm{SE}} T^{-\frac{1}{2}} \ln^{\frac{d}{4}} T \leq V_T < \tilde{C}_{\mathrm{SE}} T \ln^{\frac{d}{4}} T$, and $V_T \leq \overline{C} T^{\overline{c}}$; then, for any algorithm, there exists 
    reward functions $(f_t)_{t \in \mathbb{N}_+}$ such that Assumptions~\ref{asmp:func}, \ref{asmp:nons} and $\Ep[R_T] \geq C_{\mathrm{SE}} T^{\frac{2}{3}} V_T^{\frac{1}{3}} \ln^{\frac{d}{6}} T$ hold for sufficiently large $T$. Here, $\overline{C} > 0$ and $\overline{c} \in (0, 1)$ are any absolute constants.
    \item Suppose $k = k_{\mathrm{Mat\acute{e}rn}}$ and $\tilde{C}_{\mathrm{Mat}} T^{-\frac{\nu}{2\nu + d}} \leq V_T \leq 2^{\frac{3\nu + d}{2\nu + d}} \tilde{C}_{\mathrm{Mat}} T$; then, for any algorithm, there exists 
    reward functions $(f_t)_{t \in \mathbb{N}_+}$ such that \asmpref{asmp:func}, \ref{asmp:nons} and $\Ep[R_T] \geq C_{\mathrm{Mat}} T^{\frac{2\nu + d}{3\nu + d}} V_T^{\frac{\nu}{3\nu + d}}$ hold.
\end{itemize}
Here, $\tilde{C}_{\mathrm{SE}}, C_{\mathrm{SE}}, \tilde{C}_{\mathrm{Mat}}, C_{\mathrm{Mat}} > 0$ are constants that only depend on $\ell$, $\nu$, $B$, $\rho$, $\overline{C}$, $\overline{c}$, and $d$.
\end{theorem}

The full proof of \thmref{thm:lb} is shown in \appref{sec:proof_lb}. 
\paragraph{Proof Sketch}
The proof of \thmref{thm:lb} is derived by combining 
the idea of the lower bound of non-stationary linear bandits~\citep{cheung2019learning}
with the existing lower bound of stationary KB~\citep{scarlett2017lower}.
By following the idea of \citet{cheung2019learning}, 
we first separate the time step set $[T]$ into the $\lceil T/H \rceil$
intervals whose lengths are $H \in [T]$ (except for the last interval).
Next, we consider to assign the function $f$ that achieves
the existing lower bound of stationary KB~\citep{scarlett2017lower} 
for each interval. From this construction of $(f_t)_{t \in \mathbb{N}_+}$, $\Omega(\lceil T/H \rceil \mathrm{LB}(H))$ regret incurs even if the algorithm knows the interval length $H$. Here, $\mathrm{LB}(H)$ denotes the lower bound of stationary KB for total step size $H$. For example, $\mathrm{LB}(H) = \tilde{\Omega}(\sqrt{H})$ for $k = k_{\mathrm{SE}}$ and $\mathrm{LB}(H) = \tilde{\Omega}(H^{\frac{\nu + d}{2\nu + d}})$ for $k = k_{\text{Mat\'ern}}$.
The remaining interest is to choose the length $H$ as small as possible to maximize $\lceil T/H \rceil \mathrm{LB}(H)$ under $\sum_{t=2}^T \|f_t - f_{t-1}\|_{\infty} \leq V_T$. 
To choose such $H$, we leverage the precise characterization of the functions that achieve the lower bound provided by \citet{scarlett2017lower}.
Fortunately, with a slight modification of the proof of \citet{scarlett2017lower}, 
we can obtain the following lemma that connects the total variation of the functions 
with lower bounds.

\begin{lemma}
    \label{lem:lb_funcs}
    Suppose the same conditions as those of \thmref{thm:lb}. 
    Furthermore, for any $T \in \mathbb{N}_+$, set $\varepsilon_{\mathrm{SE}}(T) = C_{\mathrm{SE}}^{(1)}\sqrt{(\ln T)^{\frac{d}{2}}/T}$
    and $\varepsilon_{\mathrm{Mat\acute{e}rn}}(T) = C_{\mathrm{Mat}}^{(1)}T^{- \frac{\nu}{2\nu + d}}$, 
    where $C_{\mathrm{SE}}^{(1)}$, $C_{\mathrm{Mat}}^{(1)} > 0$ are constants that only depends on $\ell$, $\nu$, $\rho$, $B$, and $d$.
    Then, there exists the function set $\mF \subset \mH_{k}$ 
    such that $\forall f \in \mF, \|f\|_{\mH_k} \leq B$ and the following hold:
    \begin{itemize}
        \item For any 
        algorithm, there exist $f \in \mF$ such that 
        $\mathbb{E}[R_T] \geq T \varepsilon(T)$
        under $f_t = f$ for all $t \in \mathbb{N}_+$. Here, we 
        set $\varepsilon(T) = \varepsilon_{\mathrm{SE}}(T)$ and 
        $\varepsilon(T) = \varepsilon_{\mathrm{Mat\acute{e}rn}}(T)$
        for $k = k_{\mathrm{SE}}$ and $k = k_{\mathrm{Mat\acute{e}rn}}$, respectively.
        \item For any $f, \tilde{f} \in \mF$, $\|f - \tilde{f}\|_{\infty} \leq 4 \varepsilon(T)$.
    \end{itemize}
\end{lemma}

From \lemref{lem:lb_funcs}, the total drift $\sum_{t=2}^T \|f_t - f_{t-1}\|_{\infty}$ is bounded from above 
by $4 \lceil T/H \rceil \varepsilon(H)$ under the aforementioned construction of $(f_t)_{t \in \mathbb{N}_+}$.
Finally, by selecting the smallest $H$ such that $4 \lceil T/H \rceil \varepsilon(H) \leq V_T$ holds, the lower bounds 
$\Omega(\lceil T/H \rceil \mathrm{LB}(H))$ matches 
the results of \thmref{thm:lb}.

\paragraph{Comparison with the Lower Bounds for Stationary KB}
For standard stationary KB problems, \citet{scarlett2017lower} shows 
$\tilde{\Omega}(\sqrt{T})$ and $\Omega(T^{\frac{\nu + d}{2\nu + d}})$ lower bounds 
in SE and Mat\'ern family of kernels, respectively. By comparing 
our lower bounds, we can confirm that the learner suffers from at least $T^{\frac{1}{6}}$ and $T^{\frac{\nu^2}{(3\nu + d)(2\nu + d)}}$ additional polynomial factors at the cost of 
non-stationarity for SE and Mat\'ern family of kernels, respectively.
Note that the degeneration of the worst-case lower bounds is also observed in the existing finite-armed and linear bandit problems~\citep{besbes2014stochastic,cheung2019learning}. 
Our lower bounds show that such degeneration also occurs in KB problems.


\section{Modified Near-Optimal OPKB Algorithm for Mat\'ern Kernel}

\citet{hong2023optimization} have proposed adapting OPKB (ADA-OPKB) for the non-stationary environment without prior information of $V_T$. 
Since $V_t$ is unknown, ADA-OPKB leverages the adaptive scheduling of 
restart-reset procedures.
Then, the ADA-OPKB algorithm \citep{hong2023optimization} achieves $\tilde{O}(\gamma_T^{1/3} T^{2/3} V_T^{1/3})$ regret.
Therefore, the regret upper bounds for the SE and Mat\'ern kernels are $\tilde{O}(T^{\frac{2}{3}} V_T^{\frac{1}{3}})$ and $\tilde{O}(T^{\frac{4\nu + 3d}{6\nu + 3d}} V_T^{\frac{1}{3}})$, respectively.
Hence, for the SE kernel, the ADA-OPKB is near-optimal even though $V_T$ is unknown.
However, for the Mat\'ern kernels, the regret upper bound is worse compared with our $\Omega(T^{\frac{2\nu + d}{3\nu + d}} V_T^{\frac{\nu}{3\nu + d}})$ lower bound.
Therefore, no near-optimal regret upper bound for the Mat\'ern kernel is known.

In contrast to the unknown $V_T$ setting where \citet{hong2023optimization} focuses on, this paper focuses on the known $V_T$ setting.
In the known $V_T$ setting, by combining OPKB procedures with a restart-reset strategy whose reset interval is carefully chosen by depending on $d$, $\nu$, and $V_T$, we can achieve near-optimal $\tilde{O}(T^{\frac{2\nu + d}{3\nu + d}} V_T^{\frac{\nu}{3\nu + d}})$ regret for the Mat\'ern kernel:
%
\begin{theorem}[The modified version of OPKB algorithm with the restart-reset strategy.]
\label{thm:mod_opkb_main}
Assume that the underlying kernel is Mat\'ern kernel with smoothness parameter $\nu > 1/2$.
Furthermore, suppose that Assumptions~\ref{asmp:noise}--\ref{asmp:nons} and $V_T \geq T^{-\frac{\nu}{2\nu +d}} \ln^{\frac{\nu + d}{2\nu + d}} T$ hold. 
Then, if we set the restarting interval $H$ as $H = \lceil T^{\frac{2\nu + d}{3\nu + d}} V_T^{-\frac{2\nu + d}{3\nu + d}} \ln^{\frac{4\nu + d}{6\nu+2d}} T\rceil$, 
the OPKB algorithm (Algorithm 2 in \citet{hong2023optimization}) with the restart-reset strategy 
achieve $R_T = \tilde{\mO} (V_T^{\frac{\nu}{3\nu+d}} T^{\frac{2\nu + d}{3\nu + d}})$ with probability at least $1 - \delta$.
\end{theorem}

See \appref{sec:mod_opkb} for details.
This result also justifies the tightness of our lower bounds.
Namely, our $\Omega(V_T^{\frac{\nu}{3\nu+d}} T^{\frac{2\nu + d}{3\nu + d}})$ lower bound for the Mat\'ern kernel has no room for improvement in polynomial factors.

On the other hand, as discussed in \secref{sec:intro}, the OPKB algorithm suffers from the huge computational complexity $\mO(M|\mX|^3)$.
Therefore, in the next section, we propose yet another near-optimal algorithm, which enjoys both ease-computation and the near-optimal regret upper bound.

\section{Phased Elimination with Random Permutation}
\label{sec:rperp}

\begin{algorithm}[t!]
    \caption{Restarting Phased Elimination with Random Permutation (R-PERP)}
    \label{alg:rperp}
    \begin{algorithmic}[1]
        \REQUIRE Total step size $T$, confidence width parameter $\beta_T^{1/2} > 0$, reset interval $H \in [T]$, finite input set $\mathcal{X}$.
        \FOR {$i = 1, \ldots, \left\lceil \frac{T}{H} \right\rceil$}
            \STATE Compute the $i$-th interval size: $T^{(i)} \leftarrow \min\{H, T - (i-1)H\}$.
            \STATE Initialize the potential maximizer $\mathcal{X}_{1}^{(i)} \leftarrow \mathcal{X}$ and $N_{0}^{(i)} \leftarrow 1$.
            \FOR {$j = 1, 2, \ldots$}
                \STATE Compute the $j$-th batch size $N_{j}^{(i)} \leftarrow \min\left\{ \left\lceil \sqrt{T^{(i)} N_{j-1}^{(i)}} \right\rceil, T^{(i)} - \sum_{\tilde{j}=1}^{j-1} N_{\tilde{j}}^{(i)} \right\}$.
                \STATE Initialize the query candidate set $\mathcal{S}_{j}^{(i)} \leftarrow \emptyset$.
                \FOR {$m = 1, \ldots, N_{j}^{(i)}$}
                    \STATE $\tilde{\bm{x}}_{j, m}^{(i)} \leftarrow \mathrm{arg~max}_{\bm{x} \in \mathcal{X}_{j}^{(i)}} \sigma^2(\bm{x}; \mathcal{S}_{j}^{(i)})$.
                    \STATE $\mathcal{S}_{j}^{(i)} \leftarrow \mathcal{S}_{j}^{(i)} \cup \left\{\tilde{\bm{x}}_{j, m}^{(i)}\right\}$.
                \ENDFOR
                \STATE Obtain $\rbr{\bm{x}_{j, 1}^{(i)}, \ldots, \bm{x}_{j, N_j^{(i)}}^{(i)}}$ as an uniform permutation of $\mathcal{S}_{j}^{(i)}$.
                \FOR {$m = 1, \ldots, N_{j}^{(i)}$} 
                \STATE Observe $y_{j, m}^{(i)} = f_{j, m}^{(i)}\left(\bm{x}_{j, m}^{(i)}\right) + \epsilon_{j, m}^{(i)}$.
                \ENDFOR
                \IF {$\sum_{\tilde{j}=1}^{j} N_{\tilde{j}}^{(i)} = T^{(i)}$}
                \STATE Move to next $(i+1)$-th interval.
                \ENDIF
                \STATE $\bm{y}_j^{(i)} \leftarrow \sbr{y_{j,m}^{(i)}}_{m \leq N_j^{(i)}}$, $\bm{X}_j^{(i)} \leftarrow \sbr{\bx_{j,m}^{(i)}}_{m \leq N_j^{(i)}}$
                \STATE Calculate $\text{lcb}_{j}^{(i)}(\cdot)$ and $\text{ucb}_{j}^{(i)}(\cdot)$ as
                \begin{align}
                    \label{eq:lcb}
                    \text{lcb}_{j}^{(i)}(\bm{x}) &= \mu(\bm{x}; \bX_j^{(i)}, \by_j^{(i)}) -\beta_T^{1/2} \sigma(\bm{x}; \bX_j^{(i)}), \\
                    \label{eq:ucb}
                    \text{ucb}_{j}^{(i)}(\bm{x}) &= \mu(\bm{x}; \bX_j^{(i)}, \by_j^{(i)}) + \beta_T^{1/2} \sigma(\bm{x}; \bX_j^{(i)}).
                \end{align}
                \STATE $\mathcal{X}_{j+1}^{(i)} \leftarrow \left\{ \bm{x} \in \mathcal{X}_{j}^{(i)} ~\middle|~ \text{ucb}_{j}^{(i)}(\bm{x}) \geq \max\limits_{\tilde{\bm{x}} \in \mathcal{X}_{j}^{(i)}} \text{lcb}_{j}^{(i)}(\tilde{\bm{x}}) \right\}$
            \ENDFOR
        \ENDFOR
    \end{algorithmic}
\end{algorithm}

In this section, we show our R-PERP algorithm.
For simplicity, we assume that the input set $\mathcal{X}$ is finite without loss of generality.
That is, by relying on the discretization arguments~\citep{li2022gaussian} of the input space with the Lipschitz assumption of reward functions\footnote{Specifically, \asmpref{asmp:func} implies Lipschitz assumption when the underlying RKHS is induced by SE or Mat\'ern kernel. See, e.g., Lemma~2 in \citet{li2022gaussian}.}, the same guarantees described in this section are obtained under compact continuous input set $\mathcal{X}$.

\paragraph{Algorithm Construction}
The pseudo-code of the R-PERP algorithm is shown in \algoref{alg:rperp}. 
The algorithm construction relies on the restarting 
strategy, which reset some base KB algorithm with the prespecified interval $H$.
For the construction of the base algorithm, we leverage the PE-based algorithm~\citep{li2022gaussian}, which successively keeps and eliminates the potential maximizer based on the CBs of the reward functions.

Hereafter, we denote $\bx_{j,m}^{(i)}$ as the query point corresponding to $m$-th 
observation in $j$-th batch of PE. Furthermore, the superscript index $i$ indicate that $\bx_{j,m}^{(i)}$ is in $i$-th interval of restarting strategy.
Namely, $\bx_{j,m}^{(i)}$ denotes the query point at time step $(i-1)H + \sum_{\tilde{j} = 1}^{j-1} N_{\tilde{j}}^{(i)} + m$, where $N_{\tilde{j}}^{(i)}$ denotes 
the $\tilde{j}$-th batch size of interval $i$ (Line~5 in \algoref{alg:rperp}). 
We define $y_{j,m}^{(i)}$, $\epsilon_{j,m}^{(i)}$, and $f_{j,m}^{(i)}$ 
as the similar way to $\bx_{j,m}^{(i)}$.

\subsection{Confidence Bounds with Random Permutation}
The main challenge in adapting the PE-based algorithm is the 
construction of the CBs under non-stationary rewards.
As pointed out in Remark~1 of \citet{hong2023optimization}, the existing 
regression-based CBs~(e.g., \citet{zhou2021no}) have 
additional $\sqrt{H \gamma_H}$ factor compared with the inverse propensity score-based estimate proposed in \citet{hong2023optimization}. This $\sqrt{H\gamma_H}$ degeneration 
leads to sub-optimal regret; therefore, constructing tighter regression-based CB under non-stationarity is essential.
Our key idea is to construct CBs for the average function $\overline{f}_j^{(i)}(\cdot) \coloneqq \sum_{m=1}^{N_j^{(i)}} f_{j,m}^{(i)}(\cdot) / N_j^{(i)}$ by leveraging the random permutation of the query candidate set, instead of directly constructing CBs for some future reward functions. 

\begin{figure*}[tb]
    \centering
    \includegraphics[width=\linewidth]{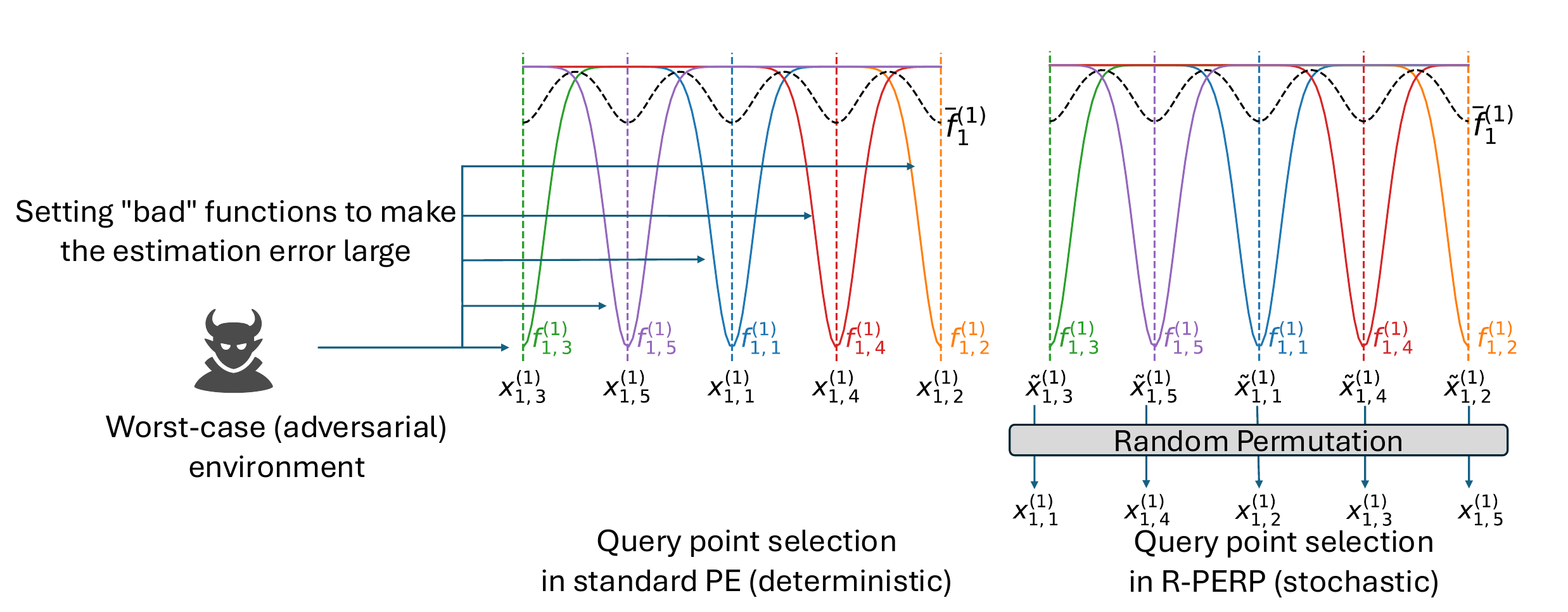}
    \caption{Illustrative image of the R-PERP algorithm in the first batch at the first interval with $N_1^{(1)} = 5$. The standard PE algorithm (left) chooses the query points deterministically within each batch. Thus, intuitively, the environment can arbitrarily choose the reward function $f_{1,j}^{1}$ such that the learner's estimation error becomes large. The R-PERP algorithm (right) alleviates the effect of such worst-case selection of the reward functions by randomly permutating the query candidates of the standard PE.}
    \label{fig:perm_image}
\end{figure*}

Here, let $\mS_j^{(i)}$ be the query candidate set at batch $j$ of interval $i$.
As with the standard stationary PE, R-PERP chooses the query candidate based on the posterior variance (Line~7--9). After that, R-PERP chooses the order of observation within 
the query candidate set by uniformly permutating the elements of $\mS_j^{(i)}$ (Line~11).
Intuitively, this permutation procedure makes it robust against the environment's worst-case reward choice. \figref{fig:perm_image} shows the illustrative image.

Based on such construction of R-PERP, the following theorem shows the tighter CB for the average function without $\sqrt{H \gamma_H}$ additional factor.

\begin{theorem}[Confidence bounds for average functions.]
    \label{thm:ave_cb}
    Fix any $T \in \mathbb{N}_+$, $H \in [T]\setminus\{1\}$, and $\delta \in (0, 1)$.
    Suppose that Assumptions~\ref{asmp:noise} and \ref{asmp:func} hold. 
    Furthermore, set the confidence width parameter $\beta_T^{1/2}$ as 
    \begin{equation}
        \label{eq:beta}
        \beta_T^{1/2} = B \rbr{\frac{C}{\sqrt{\lambda}} \sqrt{\ln \frac{4 |\mX| Q_{T,H}}{\delta}} + 1}
        + \frac{\rho}{\sqrt{\lambda}} \sqrt{2 \ln \frac{4 |\mX| Q_{T,H}}{\delta}}, 
    \end{equation}
    where $Q_{T,H} = \lceil T/H \rceil (1 + \log_2 \log_2 H)$, and $C > 0$ is an absolute constant.
    Then, when running \algoref{alg:rperp}, the following event holds with probability at least $1 - \delta$:
    \begin{equation}
    \begin{split}
        \forall i \leq \left\lceil \frac{T}{H} \right\rceil,&~\forall j \leq Q^{(i)}-1,~\forall \bx \in \mX, 
        ~\mathrm{lcb}_j^{(i)}(\bx) \leq \overline{f}_j^{(i)}(\bx) \leq \mathrm{ucb}_j^{(i)}(\bx),
    \end{split}
    \end{equation}
    where $\mathrm{lcb}_j^{(i)}(\cdot)$ and $\mathrm{ucb}_j^{(i)}(\cdot)$ 
    are defined in Eq.~\eqref{eq:lcb} and Eq.~\eqref{eq:ucb}, respectively, and $\overline{f}_j^{(i)}(\cdot) \coloneqq \sum_{m=1}^{N_j^{(i)}} f_{j,m}^{(i)}(\cdot) / N_j^{(i)}$ is the average function in $j$-th batch at $i$-th interval. Furthermore, $Q^{(i)}$ represents the total number of batches over $i$-th interval.
    
\end{theorem}
The full proof of \thmref{thm:ave_cb} is given in \appref{sec:proof_sec_rperp}.

\paragraph{Proof Sketch of \thmref{thm:ave_cb}}
We decompose the error $|\overline{f}_j^{(i)}(\bx) - \mu(\bx; \bX_j^{(i)}, \by_j^{(i)})|$ as 
\begin{align}
    \label{eq:err}
    \begin{split}
    |\overline{f}_j^{(i)}(\bx) - \mu(\bx; \bX_j^{(i)}, \by_j^{(i)})| \leq 
    |\overline{f}_j^{(i)}(\bx) - \mu(\bx; \bX_j^{(i)}, \bm{f}_j^{(i)})| + |\mu(\bx; \bX_j^{(i)}, \bm{\epsilon}_j^{(i)})|,    
    \end{split}
\end{align}
where $\bm{f}_j^{(i)} = (f_{j,1}^{(i)}(\bx_{j,1}^{(i)}), \ldots, f_{j,N_j^{(i)}}^{(i)}(\bx_{j,N_j^{(i)}}^{(i)}) )^{\top}$ and $\bm{\epsilon}_j^{(i)} = (\epsilon_{j,1}^{(i)}, \ldots, \epsilon_{j,N_j^{(i)}}^{(i)})$.
The second term in r.h.s. represents the effect of noise and is bounded from above as with the proof of Theorem~1 in \citet{vakili2021optimal}. The remaining interest is the first term. Here, let us denote $\tilde{\bx}_{j,m}^{(i)}$ as the $m$-th element of query candidate set $\mS_j^{(i)}$ (see Line~8 of \algoref{alg:rperp}). Furthermore, we denote
permutation index $\psi(m)$ as the natural number such that $\tilde{\bx}_{j,m}^{(i)} = \bx_{j,\psi(m)}^{(i)}$ holds.
Then, we obtain the equivalent expression of $\mu(\bx; \bX_j^{(i)}, \bm{f}_j^{(i)})$ as $\mu(\bx; \bX_j^{(i)}, \bm{f}_j^{(i)}) = \mu(\bx; \mS_j^{(i)}, \tilde{\bm{f}}_j^{(i)})$, where $\tilde{\bm{f}}_j^{(i)} = \left[f_{j,\psi(m)}^{(i)}\left(\tilde{\bx}_{j,m}^{(i)}\right)\right]_{m \in [N_j^{(i)}]}$. 
From this representation and the linearity of the expectation, we observe that $\Ep[\mu(\bx; \bX_j^{(i)}, \bm{f}_j^{(i)}) \mid H_{j-1}^{(i)}] = \mu(\bx; \bX_j^{(i)}, \overline{\bm{f}}_j^{(i)})$ holds, where $\overline{\bm{f}}_j^{(i)} = [\overline{f}_{j}^{(i)}(\tilde{\bx}_{j,m}^{(i)})]_{m \in [N_j^{(i)}]}$, and 
$H_{j-1}^{(i)}$ is the history up to $(j-1)$-th batch at interval $i$.
Finally, by carefully combining the above fact with the concentration inequality of random permutation (\lemref{lem:perm_conc} in the appendix) and Proposition~1 in \citet{vakili2021optimal}, we 
obtain the high-probability upper bound of the first term in Eq.~\eqref{eq:err}.

\subsection{Regret Analysis}
By combining \thmref{thm:ave_cb} with the analysis of standard PE and the gap between average function $\overline{f}_j^{(i)}$ and the true non-stationary reward functions $f_{j,m}^{(i)}$, 
we obtain the following regret upper bound of R-PERP.

\begin{theorem}[Regret upper bound of R-PERP]
    \label{thm:regret}
    Fix any $\delta \in (0, 1)$. Suppose that Assumptions~\ref{asmp:noise}--\ref{asmp:nons} hold. Furthermore, set confidence width $\beta_T^{1/2}$ as in Eq.~\eqref{eq:beta}. Then, the following statements hold with probability at least $1-\delta$:
    \begin{itemize}
        \item If $k = k_{\mathrm{SE}}$ and $V_T \geq T^{-\frac{1}{2}} \ln^{\frac{d+2}{3}} T$, the regret of R-PERP satisfies $R_T = \tilde{\mO} (V_T^{\frac{1}{3}} T^{\frac{2}{3}})$ by setting $H = \lceil T^{\frac{2}{3}} V_T^{-\frac{2}{3}} \ln^{\frac{d+2}{3}} T\rceil$.
        \item If $k = k_{\mathrm{Mat\acute{e}rn}}$ with $\nu > 1/2$ and $V_T \geq T^{-\frac{\nu}{2\nu +d}} \ln^{\frac{\nu + d}{2\nu + d}} T$, the regret of R-PERP satisfies $R_T = \tilde{\mO} (V_T^{\frac{\nu}{3\nu+d}} T^{\frac{2\nu + d}{3\nu + d}})$  by setting $H = \lceil T^{\frac{2\nu + d}{3\nu + d}} V_T^{-\frac{2\nu + d}{3\nu + d}} \ln^{\frac{4\nu + d}{6\nu+2d}} T\rceil$.
    \end{itemize}
\end{theorem}

Comparing our lower bounds (\thmref{thm:lb}) with \thmref{thm:regret}, we can confirm that R-PERP achieves nearly optimal regret for SE and Mat\'ern RKHS.

\begin{remark}
High-probability regret guarantees provided in \thmref{thm:regret} imply the expected regret guarantees of the same order. This can be easily confirmed by setting $\delta = 1/T$ and noting that $f_t(\bx^\ast) - f_t(\bx_t) \leq 2B$ holds from \asmpref{asmp:func}.
\end{remark}

\paragraph{Proof Sketch of \thmref{thm:regret}}
To leverage the regret analysis 
technique of PE, we decompose the instantaneous regret $r_{j,m}^{(i)} \coloneqq f_{j,m}^{(i)}(\bx_{j,m}^{(i)\ast}) - f_{j,m}^{(i)}(\bx_{j,m}^{(i)})$ as follows:
\begin{align}
    \begin{split}
        r_{j,m}^{(i)} 
        =
        \sbr{f_{j,m}^{(i)} \left(\bx_{j,m}^{(i)\ast}\right) - \overline{f}_{j-1}^{(i)}\left(\tilde{\bm{x}}_{j-1}^{(i)\ast}\right)}
        + \sbr{\overline{f}_{j-1}^{(i)}\left(\tilde{\bm{x}}_{j-1}^{(i)\ast}\right) - \overline{f}_{j-1}^{(i)}\left(\bm{x}_{j, m}^{(i)}\right)} 
        + \sbr{\overline{f}_{j-1}^{(i)}\left(\bm{x}_{j, m}^{(i)}\right) - f_{j,m}^{(i)}\left(\bx_{j,m}^{(i)}\right)},
    \end{split}
\end{align}
where $\tilde{\bm{x}}_{j-1}^{(i)\ast} \in \argmax_{\bx \in \mX_{j-1}^{(i)}} \overline{f}_{j-1}^{(i)}(\bx)$ is the maximum 
over the potential maximizer $\mX_{j-1}^{(i)}$ of $(j-1)$-th batch of R-PERP. 
The cumulative regret that arises from the second term is bounded 
from above by combining our CB for the average function (\thmref{thm:ave_cb}) with the analysis of the standard PE. 
In each interval, the order of this term becomes $\tilde{O}(\sqrt{H \gamma_H})$. As for the cumulative regret that arises from the first and third terms, we can obtain $\tilde{O}(V_T^{(i)} H)$ upper bound in any $i$-th interval. Here, $V_T^{(i)}$ represents the total variation of the sequence of the underlying reward functions on the $i$-th interval. 
Therefore, aggregating the cumulative regret upper bounds of each interval, we have $R_T \leq \sum_{i=1}^{\lceil T/H \rceil} \tilde{O}(V_T^{(i)} H + \sqrt{H \gamma_H}) = \tilde{O}(V_T H + T \sqrt{\gamma_H / H})$. By setting $H$ to balance $V_T H$ and $T\sqrt{\gamma_H/H}$, we obtain the desired results.

\section{Numerical Experiments}
\label{sec:exp}

\begin{figure*}[t!]
    \centering
    \includegraphics[width=0.4\linewidth]{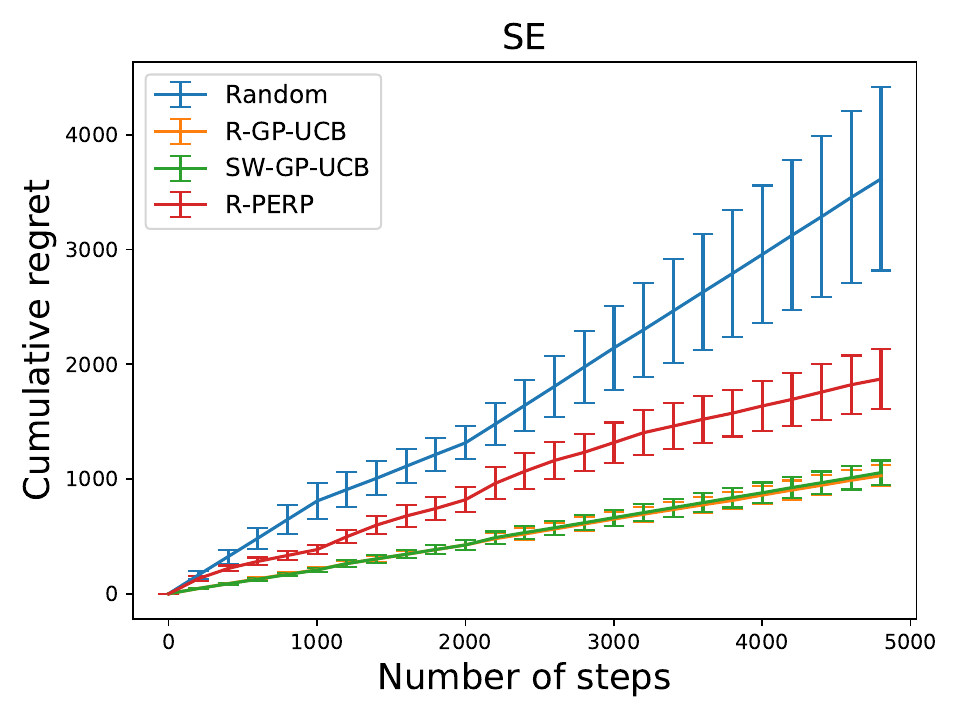}
    \includegraphics[width=0.4\linewidth]{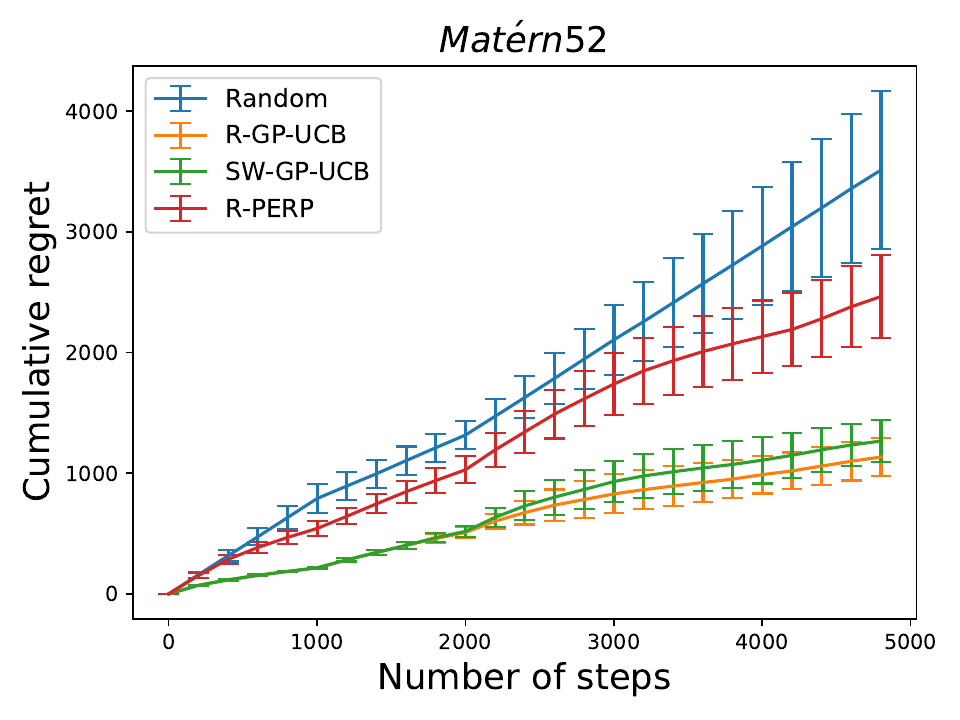}
    \caption{Numerical experiment results with $5$ different seeds. The plots show the average cumulative regret, and the error bars represent one standard error. 
    The left and right plots show the results with SE kernel and Mat\'ern kernel with $\nu = 5/2$, respectively.}
    \label{fig:experimental_result}
\end{figure*}
In this section, we confirm the empirical performance of our R-PERP algorithm. 
We would like to emphasize that we will not claim the state-of-the-art empirical performance of our R-PERP algorithm. Specifically, the existing works report that the empirical performance of PE-based algorithms tends to be inferior to UCB-like algorithms due to the large constant factor of the regret upper bound~\citep{li2022gaussian,bogunovic2022robust}. 
In our numerical experiments, we find that R-PERP inherits such deterioration of the empirical performance compared with existing UCB-like algorithms such as R-GP-UCB.
Filling the gap between practical and empirical performance is an important direction for our future work; however, we believe that the worse practical results of R-PERP do not diminish our theoretical contributions.

\paragraph{Setting}
We conduct the experiments with synthetic functions 
used in \citet{deng2022weighted}, which considers
abrupt changes in the reward functions. 
To construct such synthetic functions, we first generate
the base objective function $f_{\text{base}} \in \mH_k$ 
as $f_{\text{base}}(\bx) = \sum_{i=1}^U \alpha_i k(\cdot, \overline{\bx}_i)$, where $\alpha_i \in [-1, 1]$ 
and $\overline{\bx}_i \in [0, 1]^2$ are generated uniformly at random. Then, we define the underlying reward functions $(f_t)_{t \in [T]}$ as $f_t = f_{\text{base}}^{(1)}$ for $t \in [T/5]$; $f_t = f_{\text{base}}^{(2)}$ for $t \in [2T/5]\setminus [T/5]$; otherwise, $f_t = f_{\text{base}}^{(3)}$.
Here, $(f_{\text{base}}^{(i)})_{i \in [3]}$ are the 
base functions generated by the aforementioned procedures with different seeds. Our experiments were conducted with SE and Mat\'ern kernels with the lengthscale parameter $\ell = 0.5$ and the smoothness parameter $\nu = 5/2$. Furthermore, we set the input domain $\mathcal{X}$ as the $30 \times 30$ grid points obtained by evenly splitting $[0, 1]^2$. Finally, we set $U = 10$ and $T = 5000$.

\paragraph{Algorithms}
We consider the following four algorithms:
\begin{itemize}
    \item \textbf{Random}: Baseline algorithm whose query points are chosen uniformly at random 
    on $\mX$.
    \item \textbf{R-GP-UCB}: The GP-UCB algorithm with restart and resetting strategy~\citep{zhou2021no}. We set the restart interval $H$ based on the theoretically suggested value $H = \tilde{\gamma}_T^{1/4} (T/V_T)^{1/2}$, 
    where $\tilde{\gamma}_T = \ln^{d+1} T$ for $k = k_{\text{SE}}$ 
    and $\tilde{\gamma}_T = T^{\frac{d}{2\nu + d}} \ln^{\frac{2\nu}{2\nu + d}} T$ for $k = k_{\mathrm{Mat\acute{e}rn}}$.
    \item \textbf{SW-GP-UCB}: The sliding-window GP-UCB algorithm proposed in \citep{zhou2021no}. SW-GP-UCB constructs the GP model in each step using only past $W$ input-output pairs, where $W$ is the pre-specified window width. Then, SW-GP-UCB chooses the next query point based on the UCB score calculated by the GP model described above. As with the R-GP-UCB algorithm, we set $W$ based on theoretically suggested value $W = \tilde{\gamma}_T^{1/4} (T/V_T)^{1/2}$.
    \item \textbf{R-PERP}: Our proposed algorithm. We set the reset interval $H$ and confidence width parameter $\beta_T$ as described in \thmref{thm:regret}.
\end{itemize}
In all methods, we use the theoretically suggested confidence width $\beta_t^{1/2}$ with confidence level $\delta = 0.1$.
Here, the theoretically suggested confidence width of R-GP-UCB and SW-GP-UCB require the MIG $\gamma_T$. In our experiments, we compute the upper bound of the MIG by relying on the greedy sampling rule proposed in Section.~5.1 of \citet{srinivas10gaussian}; then, we use this upper bound as the proxy of the MIG.
Here, we would like to note that running the OPKB algorithm in our experimental setup of $\mX$ requires at least approximately $10^8$--$10^9$ order of computations; therefore, we exclude OPKB from our numerical experiments.

\paragraph{Results} 

\figref{fig:experimental_result} shows the results. 
We confirm that our R-PERP algorithm achieves superior performance to a simple random sampling baseline. Specifically, 
the average regret of R-PERP increases sub-linearly between $2000$ and $5000$ steps, even if two abrupt changes of reward functions occur at $1000$ and $2000$ steps. These results 
indicate that PE-based algorithms such as R-PERP work in non-stationary environments.
On the other hand, R-PERP has a worse practical performance than that of UCB-based algorithms. These practical performance gaps between PE-based algorithms and UCB-based algorithms in non-stationary environments are consistent with the results in the existing stationary environment settings~(e.g., \citep{bogunovic2022robust,li2022gaussian}).

\section{Conclusion}
In this paper, we study the near-optimal algorithms of the non-stationary kernelized bandit problem. We show the first lower bounds for the problem with the SE and Mat\'ern kernels.
Furthermore, we propose a novel nearly-optimal PE-based algorithm for a non-stationary environment. Our proposed algorithm is based on simple random permutation procedures on the query candidate sets of PE, which enables us to derive tighter confidence bounds tailored to non-stationary settings.

The important future direction of our research is to extend our method to 
the unknown total drift setting. Our works rely on the prior knowledge of the upper bound of the total variation $V_T$ of the underlying reward functions. Some existing works tackle the unknown $V_T$ setting by adaptively scheduling the restart-reset interval of the algorithm (e.g., \citet{wei2021non,hong2023optimization}).
Extending our proposed algorithm by incorporating their methods is an interesting and important direction for our future work.

\section*{Acknowledgements}
This work was supported by JSPS KAKENHI Grant Number (JP23K19967 and JP24K20847) and RIKEN Center for Advanced Intelligence Project.

\bibliographystyle{plainnat}
\bibliography{main}

\newpage
\appendix

\onecolumn

\section{Proof of Section~\ref{sec:lower_bound}}
\label{sec:proof_sec_lb}

\subsection{Proof of \thmref{thm:lb}}
\label{sec:proof_lb}
\begin{proof}
For any $H \in [T]$, we separate $[T]$ into the $\lceil T/H \rceil$ intervals.
Then, suppose that the length of the interval $i \in [\lceil T/H \rceil-1]$ and the 
last interval are chosen as $H$ and $T - H(\lceil T/H \rceil - 1)$, respectively.
For the notational convenience, we denote $\tilde{H}_i$ as the length of the $i$-th interval.
From \lemref{lem:lb_funcs}, given any algorithm, the sequence $(f_t)_{t \in [T]}$ 
that satisfies \ref{asmp:func} and the following properties exist:
\begin{enumerate}
    \item For any interval $i$, there exist $f \in \mF$ 
    such that $f_t = f$ for all $\sum_{j=1}^{i-1} \tilde{H}_j < t \leq \sum_{j=1}^{i} \tilde{H}_j$, where $\mF$ is the function set defined in \lemref{lem:lb_funcs}.
    \item The expected cumulative regret on the interval $i$ is at least $\tilde{H}_i \varepsilon(\tilde{H}_i)$, 
    where $\varepsilon(\cdot)$ is defined in \lemref{lem:lb_funcs}.
\end{enumerate}
Under such $(f_t)_{t \in [T]}$, we have
\begin{equation}
    \Ep[R_T] 
    \geq \sum_{i=1}^{\lceil T/H \rceil} \tilde{H}_i \varepsilon(\tilde{H}_i)
    \geq \lceil T/H \rceil H \varepsilon(H)
    \geq T \varepsilon(H).
\end{equation}
Furthermore, since the $f_t$ only changes $\lceil T/H \rceil - 1$ times from property $1$, the following 
inequality holds from \lemref{lem:lb_funcs}:
\begin{equation}
    \label{eq:upper_f_f}
    \sum_{t=2}^T \|f_t - f_{t-1}\|_{\infty}
    \leq 4 \varepsilon(H) \left(\left\lceil \frac{T}{H} \right\rceil -1\right)
    \leq 4 \varepsilon(H) \frac{T}{H}.
\end{equation}
The results in \thmref{thm:lb} are obtained by choosing $H$ such that the upper bound of \eqref{eq:upper_f_f} is equal or less than $V_T$.

\paragraph{Lower bound for the SE kernel}
From \lemref{lem:lb_funcs}, $\varepsilon(H) = C_{\mathrm{SE}}^{(1)} \sqrt{(\ln H)^{d/2} /H}$ when $k = k_{\mathrm{SE}}$. 
Then,
\begin{align}
    \sum_{t=2}^T \|f_t - f_{t-1}\|_{\infty} \leq 4 C_{\mathrm{SE}}^{(1)} (\ln T)^{d/4} H^{-3/2} T.
\end{align}
Here, we consider the setting $H = \lceil [4 C_{\mathrm{SE}}^{(1)} (\ln T)^{d/4}]^{2/3} T^{2/3} V_T^{-2/3} \rceil$
and check if the condition $4 C_{\mathrm{SE}}^{(1)} (\ln T)^{d/4} H^{-3/2} T \leq V_T$ is satisfied for such $H$. 
By aligning the condition 
$\tilde{C}_{\mathrm{SE}} T^{-1/2} \ln^{d/4} T \leq V_T < \tilde{C}_{\mathrm{SE}} T \ln^{d/4} T$ with 
$\tilde{C}_{\mathrm{SE}} = 4C_{\mathrm{SE}}^{(1)}$, we can easily 
confirm that $1 < [4 C_{\mathrm{SE}}^{(1)} (\ln T)^{d/4}]^{2/3} T^{2/3} V_T^{-2/3}$ 
and $H \leq T$ hold. Furthermore, we have 
\begin{align}
    &[4 C_{\mathrm{SE}}^{(1)} (\ln T)^{d/4}]^{2/3} T^{2/3} V_T^{-2/3} \leq H \\
    & \Leftrightarrow 4 C_{\mathrm{SE}}^{(1)} (\ln T)^{d/4} T V_T^{-1} \leq H^{3/2} \\
    & \Leftrightarrow 4 C_{\mathrm{SE}}^{(1)} (\ln T)^{d/4} H^{-3/2} T \leq V_T.
\end{align}
Therefore, for sufficiently large $T$, we have
\begin{align}
    \Ep[R_T] 
    &\geq T \varepsilon(H) \\
    &= T C_{\mathrm{SE}}^{(1)} (\ln H)^{d/4} H^{-1/2} \\
    &\geq T C_{\mathrm{SE}}^{(1)} (\ln H)^{d/4} \{2[4 C_{\mathrm{SE}}^{(1)} (\ln T)^{d/4}]^{2/3} T^{2/3} V_T^{-2/3}\}^{-1/2} \\
    &\geq T C_{\mathrm{SE}}^{(1)} \left(\ln \left\{ [4 C_{\mathrm{SE}}^{(1)} (\ln T)^{d/4}]^{2/3} \overline{C}^{-2/3} T^{2(1 - \overline{c})/3} \right\} \right)^{d/4} \{2[4 C_{\mathrm{SE}}^{(1)} (\ln T)^{d/4}]^{2/3} T^{2/3} V_T^{-2/3}\}^{-1/2} \\
    &\geq T C_{\mathrm{SE}}^{(1)} \left(-\frac{2}{3} \ln \overline{C} + \frac{2(1 - \overline{c})}{3} \ln T\right)^{d/4} 2^{-5/6} {C_{\mathrm{SE}}^{(1)}}^{-1/3} (\ln T)^{-d/12} T^{-1/3} V_T^{1/3} \\
    &\geq 2^{-5/6} {C_{\mathrm{SE}}^{(1)}}^{2/3} C_{\overline{C}, \overline{c}, d} (\ln T)^{d/6} T^{2/3} V_T^{1/3}, 
\end{align}
where $C_{\overline{C}, \overline{c}, d} > 0$ denote the constant 
that only depends on $\overline{C}, \overline{c}$, and $d$.
In the above inequalities,
\begin{itemize}
    \item the third line follows from $H \leq 2[4 C_{\mathrm{SE}}^{(1)} (\ln T)^{d/4}]^{2/3} T^{2/3} V_T^{-2/3}$ since $1 < [4 C_{\mathrm{SE}}^{(1)} (\ln T)^{d/4}]^{2/3} T^{2/3} V_T^{-2/3}$.
    \item the fourth line follows from 
    $H \geq [4 C_{\mathrm{SE}}^{(1)} (\ln T)^{d/4}]^{2/3} T^{2/3} V_T^{-2/3} \geq [4 C_{\mathrm{SE}}^{(1)} (\ln T)^{d/4}]^{2/3} \overline{C}^{-2/3} T^{2 (1 - \overline{c}) / 3}$.
    \item the fifth line follows from the fact that 
    $\ln [4 C_{\mathrm{SE}}^{(1)} (\ln T)^{d/4}]^{2/3} \geq 0$ holds for sufficiently large $T$.
    \item the last line follows from the fact 
    that there exists $C_{\overline{C}, \overline{c}, d} > 0$ 
    such that 
    $\left(-\frac{2}{3} \ln \overline{C} + \frac{2(1 - \overline{c})}{3} \ln T\right)^{d/4} \geq C_{\overline{C}, \overline{c}, d} (\ln T)^{d/4}$ for sufficiently large $T \in \mathbb{N}_+$.
\end{itemize}
Finally, defining the constant $C_{\mathrm{SE}}$ as $C_{\mathrm{SE}} = 2^{-5/6} {C_{\mathrm{SE}}^{(1)}}^{2/3} C_{\overline{C}, \overline{c}, d}$, 
we obtain the desired result.

\paragraph{Lower bound for the Mat\'ern kernel}
From \lemref{lem:lb_funcs}, $\varepsilon(H) = C_{\mathrm{Mat}}^{(1)} H^{-\frac{\nu}{2\nu + d}}$ when $k = k_{\mathrm{Mat}}$. 
Then,
\begin{align}
    \sum_{t=2}^T \|f_t - f_{t-1}\|_{\infty} \leq 4 C_{\mathrm{Mat}}^{(1)} H^{-\frac{3\nu + d}{2\nu + d}}T.
\end{align}

Here, we consider the setting $H = \lceil (4 C_{\mathrm{Mat}}^{(1)} T V_T^{-1})^{\frac{2\nu + d}{3\nu + d}} \rceil$
and check if the condition $4 C_{\mathrm{Mat}}^{(1)} H^{-\frac{3\nu + d}{2\nu + d}}T \leq V_T$ is satisfied for such $H$. 
By aligning the condition 
$\tilde{C}_{\mathrm{Mat}} T^{-\frac{\nu}{2\nu + d}} \leq V_T \leq 2^{\frac{3\nu + d}{2\nu + d}} \tilde{C}_{\mathrm{Mat}} T$ with 
$\tilde{C}_{\mathrm{Mat}} = 4C_{\mathrm{Mat}}^{(1)}$, we can easily 
confirm that $1/2 \leq (4 C_{\mathrm{Mat}}^{(1)} T V_T^{-1})^{\frac{2\nu + d}{3\nu + d}}$ 
and $H \leq T$ hold. Furthermore, we have 
\begin{align}
    &\frac{H}{2} \leq (4 C_{\mathrm{Mat}}^{(1)} T V_T^{-1})^{\frac{2\nu + d}{3\nu + d}} \leq H \\
    & \Leftrightarrow 2^{-\frac{3\nu + d}{2\nu +d}} H^{\frac{3\nu + d}{2\nu +d}} \leq 4 C_{\mathrm{Mat}}^{(1)} T V_T^{-1} \leq H^{\frac{3\nu + d}{2\nu +d}} \\
    & \Leftrightarrow 2^{-\frac{3\nu + d}{2\nu +d}} V_T \leq 4 C_{\mathrm{Mat}}^{(1)} H^{-\frac{3\nu + d}{2\nu +d}} T \leq V_T.
\end{align}
Therefore, we have
\begin{align}
    \Ep[R_T] 
    &\geq T \varepsilon(H) \\
    &= T C_{\mathrm{Mat}}^{(1)} H^{-\frac{\nu}{2\nu + d}} \\
    &= \frac{H}{4} \cdot 4C_{\mathrm{Mat}}^{(1)} H^{-\frac{3\nu + d}{2\nu + d}} T  \\
    &\geq 2^{-2 -\frac{3\nu + d}{2\nu +d}} V_T H \\
    &\geq 2^{-2 -\frac{3\nu + d}{2\nu +d}} (4 C_{\mathrm{Mat}}^{(1)})^{\frac{2\nu + d}{3\nu + d}} T^{\frac{2\nu + d}{3\nu + d}} V_T^{\frac{\nu}{3\nu + d}}.
\end{align}
Defining the constant $C_{\mathrm{Mat}}$ as $C_{\mathrm{Mat}} = 2^{-2 -\frac{3\nu + d}{2\nu +d}} (4 C_{\mathrm{Mat}}^{(1)})^{\frac{2\nu + d}{3\nu + d}}$, 
we obtain the desired result.
\end{proof}

\subsection{Proof of \lemref{lem:lb_funcs}}
\label{sec:proof_lb_funcs}
\lemref{lem:lb_funcs} can be obtained with minor modification of the proof of Theorem~2 in \citet{scarlett2017lower}. The required modifications are as follows:

\begin{itemize}
\item The upper bound of the total reward (Eq.~(70) on p16 in \citet{scarlett2017lower}). In the original paper, the upper bound of Eq.~(70) is given as the worst-case total reward for the elements over $\mathcal{H}_k(B) \coloneqq \{f \in \mathcal{H}_k \mid \|f\|_{\mH_k} \leq B\}$. This upper bound is derived from the upper bound of the average reward of a finite function family $\mathcal{F}$, which is constructed by the Fourier transform of bump functions (for details on $\mathcal{F}$, refer to Sec.~3 of \citet{scarlett2017lower} or Sec.~4 of \citet{cai2021lower}). Since the upper bound on the average reward over $\mathcal{F}$ implies that the upper bound holds for some element in $\mathcal{F}$, we can prove the lower bound of regret for such $f$ in the same manner as in the original paper.

\item The output range of an element of aforementioned $\mathcal{F}$ is, due to its construction, bounded in $[-2\varepsilon(T), 2\varepsilon(T)]$. From this property, it is trivial that for any $f, \tilde{f} \in \mathcal{F}$, $\|f - \tilde{f}\| \leq 4\varepsilon(T)$ holds. Moreover, the constructions of $\varepsilon(T)$ for SE kernels and Mat\'ern kernels are described on $p17$ in \citet{scarlett2017lower}.
\end{itemize}

\section{Proof of Section~\ref{sec:rperp}}
\label{sec:proof_sec_rperp}
\subsection{Proof of \thmref{thm:ave_cb}}
\begin{lemma}
    \label{lem:eperm}
    Fix any $i \in \lceil T/H \rceil$, $j \leq Q^{(i)} -1$, 
    and $\bx \in \mX$, where $Q^{(i)}$ is the total number of 
    batch on $i$-th interval. Under \asmpref{asmp:func}, the following inequality holds when running \algoref{alg:rperp}:
    \begin{equation}
        \left| \Ep\sbr{\overline{f}_j^{(i)}(\bx) - \mu\rbr{\bx; \bX_j^{(i)}, \bm{f}_j^{(i)}} ~\middle|~ H_{j-1}^{(i)}} \right| \leq B \sigma\rbr{\bx; \bX_j^{(i)}},
    \end{equation}
    where $H_{j-1}^{(i)} \coloneqq \bigcup_{\tilde{i}\leq i, \tilde{j} \leq j-1} \{(\bx_{\tilde{j}, m}^{(\tilde{i})}, y_{\tilde{j}, m}^{(\tilde{i})})\}_{m \leq N_{\tilde{j}}^{\tilde{i}}}$ represents the history up to $j-1$-th batch on $i$-th interval.
\end{lemma}
\begin{proof}
Note that the query candidate points $\tilde{\bx}_{j,m}^{(i)}$ are fixed 
given $H_{j-1}^{(i)}$. 
Furthermore, since the permutation index is chosen uniformly, 
we have $\Ep[f_{j, \psi(m)}^{(i)}(\tilde{\bx}_{j, m}^{(i)})] = \overline{f}_j^{(i)}(\tilde{\bx}_{j, m}^{(i)})$. 
Therefore,
\begin{align}
    &\Ep\sbr{\overline{f}_j^{(i)}(\bx) - \mu\rbr{\bx; \bX_j^{(i)}, \bm{f}_j^{(i)}} ~\middle|~ H_{j-1}^{(i)}} \\
    &= \overline{f}_j^{(i)}(\bx) - \Ep\sbr{\mu\rbr{\bx; \mS_j^{(i)}, \tilde{\bm{f}}_j^{(i)}} ~\middle|~ H_{j-1}^{(i)}} \\
    &= \overline{f}_j^{(i)}(\bx) - \bk(\bx, \mS_j^{(i)})^{\top} (\bK(\mS_j^{(i)}, \mS_j^{(i)}) + \lambda \bI_t)^{-1} \Ep\sbr{ \tilde{\bm{f}}_j^{(i)} ~\middle|~ H_{j-1}^{(i)}} \\
    &= \overline{f}_j^{(i)}(\bx) - \bk(\bx, \mS_j^{(i)})^{\top} (\bK(\mS_j^{(i)}, \mS_j^{(i)}) + \lambda \bI_t)^{-1} \overline{\bm{f}}_j^{(i)} \\
    &= \overline{f}_j^{(i)}(\bx) - \mu\rbr{\bx; \mS_j^{(i)}, \overline{\bm{f}}_j^{(i)}},
\end{align}
where $\tilde{\bm{f}}_j^{(i)} = \left[f_{j,\psi(m)}^{(i)}\left(\tilde{\bx}_{j,m}^{(i)}\right)\right]_{m \in [N_j^{(i)}]}$ and $\overline{\bm{f}}_j^{(i)} = [\overline{f}_{j}^{(i)}(\tilde{\bx}_{j,m}^{(i)})]_{m \in [N_j^{(i)}]}$.
Finally, by combining Proposition~1 in \citet{vakili2021optimal} with 
$\|\overline{f}_j^{(i)}\|_{\mH_k} \leq B$, we obtain
\begin{align}
    \left| \Ep\sbr{\overline{f}_j^{(i)}(\bx) - \mu\rbr{\bx; \bX_j^{(i)}, \bm{f}_j^{(i)}} ~\middle|~ H_{j-1}^{(i)}} \right| \leq B \sigma\rbr{\bx; \mS_j^{(i)}} = B \sigma\rbr{\bx; \bX_j^{(i)}}.
\end{align}
\end{proof}

\begin{lemma}[Theorem~3.1 in \citet{adamczak2016circular}]
    \label{lem:perm_conc}
    Fix any sequence $a_1, \ldots, a_n \in [0, 1]$. 
    Suppose that $h: [0, 1]^{n} \rightarrow \mathbb{R}$ 
    is an $L$-Lipschitz convex function. Then, 
    the following inequality holds for any $\eta \geq 0$:
    \begin{equation}
        \Pr(|h(x_{\psi(1)}, \ldots, x_{\psi(n)}) - \Ep[h(x_{\psi(1)}, \ldots, x_{\psi(n)})]| \geq \eta) \leq 2 \exp\rbr{-\frac{c \eta^2}{L^2}},
    \end{equation}
    where $c > 0$ is an absolute constant. 
    Furthermore, $\psi(\cdot)$ represent the uniform 
    permutation indices on the set $[n]$.
\end{lemma}

\begin{lemma}
    \label{lem:nl_err}
    Fix any $i \in \lceil T/H \rceil$, $j \leq Q^{(i)} -1$, $\bx \in \mX$, and $\delta \in (0, 1)$. Then, under Assumptions~\ref{asmp:noise}, \ref{asmp:func}, the following inequality holds with probability at least $1 - \delta$ when running \algoref{alg:rperp}:
    \begin{equation}
        \left|\overline{f}_j^{(i)}(\bx) - \mu\rbr{\bx; \bX_j^{(i)}, \bm{f}_j^{(i)}}\right| \leq B \rbr{1 + \frac{C}{\sqrt{\lambda}} \sqrt{\ln \frac{2}{\delta}}} \sigma\rbr{\bx; \bX_j^{(i)}},
    \end{equation}
    where $C > 0$ is an absolute constant.
\end{lemma}
\begin{proof}
    Let us define the function $h: [0, 1]^{N_j^{(i)}} \rightarrow \R$
    as $h(a_1, \ldots, a_{N_j^{(i)}}) = \bk(\bx, \mS_j^{(i)})^{\top} (\bK(\mS_j^{(i)}, \mS_j^{(i)}) + \lambda \bI_t)^{-1} (a_1, \ldots, a_{N_j^{(i)}})^{\top}$. Then, given the history $H_{j-1}^{(i)}$, 
    the function $h$ is a fixed convex function. Furthermore, 
    from Proposition~1 in \citet{vakili2021optimal}, 
    the following inequality holds for any $\ba^{(1)}, \ba^{(2)} \in [0, 1]^{N_j^{(i)}}$:
    \begin{align}
        |h(\ba^{(1)}) - h(\ba^{(2)})|
        &= |\bk(\bx, \mS_j^{(i)})^{\top} (\bK(\mS_j^{(i)}, \mS_j^{(i)}) + \lambda \bI_t)^{-1}(\ba^{(1)} - \ba^{(2)})| \\
        &\leq \|\bk(\bx, \mS_j^{(i)})^{\top} (\bK(\mS_j^{(i)}, \mS_j^{(i)}) + \lambda \bI_t)^{-1} \|_2 \|\ba^{(1)} - \ba^{(2)} \|_2 \\
        &\leq \lambda^{-1/2} \sigma(\bx; \bX_j^{(i)}) \|\ba^{(1)} - \ba^{(2)} \|_2.
    \end{align}
    Therefore, given the history $H_{j-1}^{(i)}$, $h$ is an $\lambda^{-1/2} \sigma(\bx; \bX_j^{(i)})$-Lipschitz convex function.
    Here, for any $H_{j-1}^{(i)}$ and $\eta \geq 0$, we have
    \begin{align}
        &\left| \overline{f}_j^{(i)}(\bx) - \mu\rbr{\bx; \bX_j^{(i)}, \bm{f}_j^{(i)}} - \Ep\sbr{\overline{f}_j^{(i)}(\bx) - \mu\rbr{\bx; \bX_j^{(i)}, \bm{f}_j^{(i)}} ~\middle|~ H_{j-1}^{(i)}} \right| \geq \eta \\
        &\Leftrightarrow \left| \mu\rbr{\bx; \mS_j^{(i)}, \tilde{\bm{f}}_j^{(i)}} - \Ep\sbr{\mu\rbr{\bx; \mS_j^{(i)}, \tilde{\bm{f}}_j^{(i)}} ~\middle|~ H_{j-1}^{(i)}} \right| \geq \eta \\
        &\Leftrightarrow \left| \mu\rbr{\bx; \mS_j^{(i)}, \frac{\tilde{\bm{f}}_j^{(i)} + B}{2B}} - \Ep\sbr{\mu\rbr{\bx; \mS_j^{(i)}, \frac{\tilde{\bm{f}}_j^{(i)} + B}{2B}} ~\middle|~ H_{j-1}^{(i)}} \right| \geq \frac{\eta}{2B} \\
        &\Leftrightarrow \left| h\rbr{\frac{\tilde{\bm{f}}_j^{(i)} + B}{2B}} - \Ep\sbr{h\rbr{\frac{\tilde{\bm{f}}_j^{(i)} + B}{2B}} ~\middle|~ H_{j-1}^{(i)}}\right| \geq \frac{\eta}{2B}.
    \end{align}
    By noting $\frac{\tilde{f}_{j,m}^{(i)} + B}{2B} \in [0, 1]$, we have 
    the following inequality from \lemref{lem:perm_conc}:
    \begin{align}
        &\Pr\rbr{\left| \overline{f}_j^{(i)}(\bx) - \mu\rbr{\bx; \bX_j^{(i)}, \bm{f}_j^{(i)}} - \Ep\sbr{\overline{f}_j^{(i)}(\bx) - \mu\rbr{\bx; \bX_j^{(i)}, \bm{f}_j^{(i)}} ~\middle|~ H_{j-1}^{(i)}} \right| \geq \eta ~\middle|~ H_{j-1}^{(i)}} \\
        & \leq 2 \exp\rbr{-\frac{c\eta^2}{4B^2 \lambda^{-1} \sigma^2(\bx; \bX_j^{(i)})}}.
    \end{align}
    Setting $\eta$ as $\eta = 2B \lambda^{-1/2} \sigma(\bx; \bX_j^{(i)}) \sqrt{c^{-1} \ln(2/\delta)}$ in the above inequality, we obtain
    \begin{equation}
        \Pr\rbr{\left| \overline{f}_j^{(i)}(\bx) - \mu\rbr{\bx; \bX_j^{(i)}, \bm{f}_j^{(i)}} - \Ep\sbr{\overline{f}_j^{(i)}(\bx) - \mu\rbr{\bx; \bX_j^{(i)}, \bm{f}_j^{(i)}} ~\middle|~ H_{j-1}^{(i)}} \right| < \eta ~\middle|~ H_{j-1}^{(i)}} 
         \geq 1 - \delta.
    \end{equation}
    Here, by combining the above inequality with \lemref{lem:eperm}, we have
    \begin{align}
        \Pr\rbr{\left| \overline{f}_j^{(i)}(\bx) - \mu\rbr{\bx; \bX_j^{(i)}, \bm{f}_j^{(i)}}\right| < B \rbr{ 1 + 2\lambda^{-1/2} \sqrt{c^{-1} \ln(2/\delta)}}\sigma(\bx; \bX_j^{(i)}) ~\middle|~ H_{j-1}^{(i)}} \geq 1 - \delta.
    \end{align}
    From the tower property of the conditional expectation, 
    we obtain the desired result by setting the absolute constant $C$ as $C = 2 \sqrt{c^{-1}}$.
\end{proof}

Now, we describe the proof of \thmref{thm:ave_cb}.

\begin{proof}[Proof of \thmref{thm:ave_cb}]
    From \lemref{lem:nl_err} and the union bound, with probability at least $1 - \delta/2$, the following inequality holds for any $i \leq \lceil T/H \rceil$, $j \leq Q^{(i)} - 1$, and $\bx \in \mX$:
    \begin{equation}
        \left| \overline{f}_j^{(i)}(\bx) - \mu\rbr{\bx; \bX_j^{(i)}, \bm{f}_j^{(i)}}\right| < B \rbr{ 1 + C\lambda^{-1/2} \sqrt{\ln(4 |\mX| \tilde{Q}_{T,H}/\delta)}}\sigma(\bx; \bX_j^{(i)}),
    \end{equation}
    where $\tilde{Q}_{T,H} = \sum_{i=1}^{\lceil T/H \rceil} (Q^{(i)} - 1)$. 
    Furthermore, as with the proof of Theorem~1 in \citet{vakili2021optimal}, 
    with probability at least $1 - \delta/2$, the following inequality holds for any $i \leq \lceil T/H \rceil$, $j \leq Q^{(i)} - 1$, and $\bx \in \mX$:
    \begin{equation}
        \left| \mu(\bx; \bX_j^{(i)}, \bm{\epsilon}_j^{(i)}) \right| < \frac{\rho}{\sqrt{\lambda}} \sqrt{2 \ln \frac{4 |\mX| \tilde{Q}_{T,H}}{\delta}}\sigma(\bx; \bX_j^{(i)}),
    \end{equation}
    Finally, by noting that $\tilde{Q}_{T,H} \leq Q_{T,H}$ holds from Proposition~1 in \citet{li2022gaussian}, we have the desired result.
    
\end{proof}

\subsection{Proof of \thmref{thm:regret}}
\begin{lemma}
    \label{lem:fx_fox}
    For any $i \leq \lceil T/H \rceil $, $j \leq Q^{(i)}$, and $m \leq N_j^{(i)}$, the following inequality holds:
    \begin{equation}
        f_{j, m}^{(i)}\left(\bm{x}_{j, m}^{(i)\ast}\right) - f_{j, m}^{(i)}\left(\overline{\bm{x}}_{j}^{(i)\ast}\right)
        \leq 2V_T^{(i)},
    \end{equation}
    where $\bx_{j,m}^{(i)\ast} \in \mathrm{arg~max}_{\bx \in \mX} f_{j,m}^{(i)}(\bx)$, $\overline{\bm{x}}_{j}^{(i)\ast} \in \mathrm{arg~max}_{\bx \in \mX} \overline{f}_{j}^{(i)}(\bx)$, 
    and $V_T^{(i)} = \sum_{t=(i-1)H+1}^{iH-1} \|f_{t+1} - f_t\|_{\infty}$
\end{lemma}
\begin{proof}
    We prove this by contradiction. Assume that $m^{\ast} \in \mathrm{arg~max}_{m \in \left[N_j^{(i)}\right]} \left[ f_{j, m}^{(i)}\left(\bm{x}_{j, m}^{(i)\ast}\right) - f_{j, m}^{(i)}\left(\overline{\bm{x}}_{j}^{(i)\ast}\right) \right]$ and that $f_{j, m^{\ast}}^{(i)}\left(\bm{x}_{j, m^{\ast}}^{(i)\ast}\right) - f_{j, m^{\ast}}^{(i)}\left(\overline{\bm{x}}_{j}^{(i)\ast}\right) > 2V_T^{(i)}$ holds. Then, for any $m \in \left[N_j^{(i)}\right]$, we have:
    \begin{align}
        f_{j, m}^{(i)} \left( \bm{x}_{j, m^{\ast}}^{(i)\ast} \right)
        &\geq f_{j, m^{\ast}}^{(i)} \left( \bm{x}_{j, m^{\ast}}^{(i)\ast} \right) - V_T^{(i)} \\
        &> f_{j, m^{\ast}}^{(i)}\left(\overline{\bm{x}}_{j}^{(i)\ast}\right) + V_T^{(i)} \\
        &\geq f_{j, m}^{(i)}\left(\overline{\bm{x}}_{j}^{(i)\ast}\right).
    \end{align}
    Therefore,
    \begin{equation}
        \overline{f}_{j}^{(i)} \left( \bm{x}_{j, m^{\ast}}^{(i)\ast} \right)
        > \overline{f}_{j}^{(i)}\left(\overline{\bm{x}}_{j}^{(i)\ast}\right).
    \end{equation}
    This contradicts the definition of $\overline{\bm{x}}_{j}^{(i)\ast}$.
\end{proof}

\begin{lemma}
    For any $i \leq \lceil T/H \rceil $, $j \leq Q^{(i)}$, and $m \leq N_j^{(i)}$, the following inequality holds:
    \begin{equation}
        f_{j, m}^{(i)}\left(\overline{\bm{x}}_{j}^{(i)\ast}\right) - \overline{f}_{j}^{(i)}\left(\overline{\bm{x}}_{j}^{(i)\ast}\right)
        \leq V_T^{(i)}.
    \end{equation}
\end{lemma}
\begin{proof}
    \begin{align}
        f_{j, m}^{(i)}\left(\overline{\bm{x}}_{j}^{(i)\ast}\right) - \overline{f}_{j}^{(i)}\left(\overline{\bm{x}}_{j}^{(i)\ast}\right)
        &= \frac{1}{N_j^{(i)}} \sum_{\tilde{m}=1}^{N_j^{(i)}} \left[f_{j, m}^{(i)}\left(\overline{\bm{x}}_{j}^{(i)\ast}\right) - f_{j, \tilde{m}}^{(i)}\left(\overline{\bm{x}}_{j}^{(i)\ast}\right) \right] \\
        &\leq \frac{1}{N_j^{(i)}} \sum_{\tilde{m}=1}^{N_j^{(i)}} V_T^{(i)} \\
        &= V_T^{(i)}.
    \end{align}
\end{proof}

\begin{lemma}
    For any $i \leq \lceil T/H \rceil $ and $j \left(2 \leq j \leq Q^{(i)}\right)$, the following inequality holds:
    \begin{equation}
        \overline{f}_{j}^{(i)}\left(\overline{\bm{x}}_{j}^{(i)\ast}\right) - \overline{f}_{j-1}^{(i)}\left(\overline{\bm{x}}_{j-1}^{(i)\ast}\right)
        \leq V_T^{(i)}.
    \end{equation}
\end{lemma}
\begin{proof}
    Let $m^{\ast} = \mathrm{arg~max}_{m \in [N_j^{(i)}]}f_{j, m}^{(i)}\left(\overline{\bm{x}}_{j}^{(i)\ast}\right)$, then
    \begin{align}
        \overline{f}_{j}^{(i)}\left(\overline{\bm{x}}_{j}^{(i)\ast}\right) - \overline{f}_{j-1}^{(i)}\left(\overline{\bm{x}}_{j-1}^{(i)\ast}\right)
        &\leq \overline{f}_{j}^{(i)}\left(\overline{\bm{x}}_{j}^{(i)\ast}\right) - \overline{f}_{j-1}^{(i)}\left(\overline{\bm{x}}_{j}^{(i)\ast}\right) \\
        &= \frac{1}{N_{j-1}^{(i)}} \sum_{\tilde{m}=1}^{N_{j-1}^{(i)}} \left[\overline{f}_{j}^{(i)}\left(\overline{\bm{x}}_{j}^{(i)\ast}\right) - f_{j-1, \tilde{m}}^{(i)}\left(\overline{\bm{x}}_{j}^{(i)\ast}\right) \right] \\
        &\leq \frac{1}{N_{j-1}^{(i)}} \sum_{\tilde{m}=1}^{N_{j-1}^{(i)}} \left[f_{j, m^{\ast}}^{(i)}\left(\overline{\bm{x}}_{j}^{(i)\ast}\right) - f_{j-1, \tilde{m}}^{(i)}\left(\overline{\bm{x}}_{j}^{(i)\ast}\right) \right] \\
        &\leq \frac{1}{N_{j-1}^{(i)}} \sum_{\tilde{m}=1}^{N_{j-1}^{(i)}} V_T^{(i)} \\
        &= V_T^{(i)}.
    \end{align}
\end{proof}

\begin{lemma}
    \label{lem:of_oftx}
    Suppose the following event holds:
    \begin{equation}
        \label{eq:conf_event}
        \forall i \leq \left\lceil \frac{T}{H} \right\rceil, \forall j \leq Q^{(i)} - 1, \forall \bm{x} \in \mathcal{X}, 
        \mathrm{lcb}_j^{(i)}(\bm{x}) \leq \overline{f}_{j}^{(i)}(\bm{x}) \leq \mathrm{ucb}_j^{(i)}(\bm{x}).
    \end{equation}
    Then, for any $i \leq \lceil T/H \rceil $, $j \left(2 \leq j \leq Q^{(i)}\right)$, and $m \leq N_j^{(i)}$, we have
    \begin{equation}
        \overline{f}_{j-1}^{(i)}\left(\overline{\bm{x}}_{j-1}^{(i)\ast}\right) - \overline{f}_{j-1}^{(i)}\left(\tilde{\bm{x}}_{j-1}^{(i)\ast}\right)
        \leq \left( \log_2 \log_2 T^{(i)} + 1 \right) V_T^{(i)},
    \end{equation}
    where $\tilde{\bx}_j^{(i)\ast} \in \mathrm{arg~max}_{\bx \in \mX_j^{(i)}} \overline{f}_j^{(i)}(\bx)$.
\end{lemma}
\begin{proof}
    The case where $\overline{\bm{x}}_{j-1}^{(i)\ast} = \tilde{\bm{x}}_{j-1}^{(i)\ast}$ is trivial.
    When $\overline{\bm{x}}_{j-1}^{(i)\ast} \neq \tilde{\bm{x}}_{j-1}^{(i)\ast}$, there exists some $\tilde{j} < j - 1$ such that
    \begin{equation}
        \overline{\bm{x}}_{j-1}^{(i)\ast} \in \mathcal{X}_{\tilde{j}}^{(i)}~~\text{and}~~\overline{\bm{x}}_{j-1}^{(i)\ast} \notin \mathcal{X}_{\tilde{j}+1}^{(i)}.
    \end{equation}
    In this case, we have
    \begin{align}
        &\overline{f}_{j-1}^{(i)}\left(\overline{\bm{x}}_{j-1}^{(i)\ast}\right) - \overline{f}_{j-1}^{(i)}\left(\tilde{\bm{x}}_{j-1}^{(i)\ast}\right) \\
        &= \overline{f}_{j-1}^{(i)}\left(\overline{\bm{x}}_{j-1}^{(i)\ast}\right) - \overline{f}_{\tilde{j}}^{(i)}\left(\overline{\bm{x}}_{j-1}^{(i)\ast}\right) 
        + \overline{f}_{\tilde{j}}^{(i)}\left(\overline{\bm{x}}_{j-1}^{(i)\ast}\right) - \overline{f}_{\tilde{j}}^{(i)}\left(\tilde{\bm{x}}_{\tilde{j}}^{(i)\ast}\right) 
        + \overline{f}_{\tilde{j}}^{(i)}\left(\tilde{\bm{x}}_{\tilde{j}}^{(i)\ast}\right) - \overline{f}_{j-1}^{(i)}\left(\tilde{\bm{x}}_{j-1}^{(i)\ast}\right) \\
        &\leq V_T^{(i)} + \sum_{\hat{j}=\tilde{j}}^{j-2} \left[ \overline{f}_{\hat{j}}^{(i)}\left(\tilde{\bm{x}}_{\hat{j}}^{(i)\ast}\right) - \overline{f}_{\hat{j}+1}^{(i)}\left(\tilde{\bm{x}}_{\hat{j}+1}^{(i)\ast}\right)\right].
    \end{align}
    In the final line, we use the fact that $\overline{f}_{j-1}^{(i)}\left(\overline{\bm{x}}_{j-1}^{(i)\ast}\right) - \overline{f}_{\tilde{j}}^{(i)}\left(\overline{\bm{x}}_{j-1}^{(i)\ast}\right) \leq V_T^{(i)}$, and since $\overline{\bm{x}}_{j-1}^{(i)\ast} \in \mathcal{X}_{\tilde{j}}^{(i)}$, it follows that $\overline{f}_{\tilde{j}}^{(i)}\left(\overline{\bm{x}}_{j-1}^{(i)\ast}\right) \leq \overline{f}_{\tilde{j}}^{(i)}\left(\tilde{\bm{x}}_{\tilde{j}}^{(i)\ast}\right)$.
    Under the event \eqref{eq:conf_event}, for any $\hat{j}$, we have $\tilde{\bm{x}}_{\hat{j}}^{(i)\ast} \in \mathcal{X}_{\hat{j} + 1}^{(i)}$.
    Therefore, 
    \begin{align}
        &\sum_{\hat{j}=\tilde{j}}^{j-2} \left[ \overline{f}_{\hat{j}}^{(i)}\left(\tilde{\bm{x}}_{\hat{j}}^{(i)\ast}\right) - \overline{f}_{\hat{j}+1}^{(i)}\left(\tilde{\bm{x}}_{\hat{j}+1}^{(i)\ast}\right)\right] \\
        &= \sum_{\hat{j}=\tilde{j}}^{j-2} \left[ \overline{f}_{\hat{j}}^{(i)}\left(\tilde{\bm{x}}_{\hat{j}}^{(i)\ast}\right) - \overline{f}_{\hat{j}+1}^{(i)}\left(\tilde{\bm{x}}_{\hat{j}}^{(i)\ast}\right) +  \overline{f}_{\hat{j}+1}^{(i)}\left(\tilde{\bm{x}}_{\hat{j}}^{(i)\ast}\right) - \overline{f}_{\hat{j}+1}^{(i)}\left(\tilde{\bm{x}}_{\hat{j}+1}^{(i)\ast}\right)\right] \\
        &\leq \sum_{\hat{j}=\tilde{j}}^{j-2} V_T^{(i)} \\
        &\leq (Q^{(i)} - 2) V_T^{(i)} \\
        &\leq  V_T^{(i)} \log_2 \log_2 T^{(i)}.
    \end{align}
    The final line uses Proposition~1 from \cite{li2022gaussian}.
\end{proof}

\begin{lemma}
    \label{lem:of_f}
    For any $i \leq \lceil T/H \rceil $, $j \left(2 \leq j \leq Q^{(i)}\right)$, and $m \leq N_j^{(i)}$, the following inequality holds:
    \begin{equation}
        \overline{f}_{j-1}^{(i)}\left(\bm{x}_{j, m}^{(i)}\right) - f_{j, m}^{(i)}\left(\bm{x}_{j, m}^{(i)}\right)
        \leq V_T^{(i)}.
    \end{equation}
\end{lemma}
\begin{proof}
    \begin{align}
        \overline{f}_{j-1}^{(i)}\left(\bm{x}_{j, m}^{(i)}\right) - f_{j, m}^{(i)}\left(\bm{x}_{j, m}^{(i)}\right)
        &\leq N_{j-1}^{(i)-1} \sum_{\tilde{m}=1}^{N_{j-1}^{(i)}} \left[ f_{j, \tilde{m}}^{(i)}\left(\bm{x}_{j, m}^{(i)}\right) - f_{j, m}^{(i)}\left(\bm{x}_{j, m}^{(i)}\right) \right] \\
        &\leq N_{j-1}^{(i)-1} \sum_{\tilde{m}=1}^{N_{j-1}^{(i)}} V_T^{(i)} \\
        &\leq V_T^{(i)}.
    \end{align}
\end{proof}

\begin{lemma}
    \label{lem:pe_bound}
    Suppose that the event \eqref{eq:conf_event} holds.
    Then, for any $i \leq \lceil T/H \rceil$, the following inequality holds:
    \begin{equation}
        \sum_{j=2}^{Q^{(i)}} \sum_{m=1}^{N_{j}^{(i)}} \left[ \overline{f}_{j-1}^{(i)}\left(\tilde{\bm{x}}_{j-1}^{(i)\ast}\right) - \overline{f}_{j-1}^{(i)}\left(\bm{x}_{j, m}^{(i)}\right) \right]
        \leq 4 \left(\log_2 \log_2 T^{(i)} + 1\right) \left( \sqrt{T^{(i)}} + T^{(i)-1/4}\right) \sqrt{C_1 \gamma_{T^{(i)}} \beta_T}, 
    \end{equation}
    where $C_1 = 8/\ln (1 + \sigma^{-2})$. 
\end{lemma}
\begin{proof}
    Under the event \eqref{eq:conf_event},
    \begin{align}
        &\overline{f}_{j-1}^{(i)}\left(\tilde{\bm{x}}_{j-1}^{(i)\ast}\right) - \overline{f}_{j-1}^{(i)}\left(\bm{x}_{j, m}^{(i)}\right) \\
        &\leq \text{ucb}_{j-1}^{(i)}\left(\tilde{\bm{x}}_{j-1}^{(i)\ast}\right) - \text{lcb}_{j-1}^{(i)}\left(\bm{x}_{j, m}^{(i)}\right) \\
        &= \text{ucb}_{j-1}^{(i)}\left(\tilde{\bm{x}}_{j-1}^{(i)\ast}\right) - \text{ucb}_{j-1}^{(i)}\left(\bm{x}_{j, m}^{(i)}\right) + 2\beta_T^{1/2} \sigma \left(\bm{x}_{j, m}^{(i)}; \bX_{j-1}^{(i)}\right)  \\
        &\leq \text{ucb}_{j-1}^{(i)}\left(\tilde{\bm{x}}_{j-1}^{(i)\ast}\right) - \max_{\bm{x} \in \mathcal{X}_{j-1}^{(i)}} \text{lcb}_{j-1}^{(i)}\left(\bm{x}\right) + 2\beta_T^{1/2} \sigma \left(\bm{x}_{j, m}^{(i)}; \bX_{j-1}^{(i)}\right)\\
        &\leq \text{ucb}_{j-1}^{(i)}\left(\tilde{\bm{x}}_{j-1}^{(i)\ast}\right) - \text{lcb}_{j-1}^{(i)}\left(\tilde{\bm{x}}_{j-1}^{(i)\ast}\right) + 2\beta_T^{1/2} \sigma \left(\bm{x}_{j, m}^{(i)}; \bX_{j-1}^{(i)}\right) \\
        &\leq 4 \beta_T^{1/2} \max_{\bm{x} \in \mathcal{X}_{j-1}^{(i)}} \sigma \left(\bm{x}; \bX_{j-1}^{(i)}\right) \\
        &\leq 4 \sqrt{\frac{C_1 \beta_T \gamma_{N_{j-1}^{(i)}}}{N_{j-1}^{(i)}}} \\
        &\leq 4 \sqrt{\frac{C_1 \beta_T \gamma_{T^{(i)}}}{N_{j-1}^{(i)}}}.
    \end{align}
    Therefore, 
    \begin{align}
        &\sum_{j=2}^{Q^{(i)}} \sum_{m=1}^{N_{j}^{(i)}} \left[ \overline{f}_{j-1}^{(i)}\left(\tilde{\bm{x}}_{j-1}^{(i)\ast}\right) - \overline{f}_{j-1}^{(i)}\left(\bm{x}_{j, m}^{(i)}\right) \right] \\
        &\leq 4 \sum_{j=2}^{Q^{(i)}} N_{j}^{(i)} \sqrt{\frac{C_1 \beta_T \gamma_{T^{(i)}}}{N_{j-1}^{(i)}}} \\
        &\leq 4 \left(Q^{(i)} - 1\right) \left(\sqrt{T^{(i)}} + T^{(i)-1/4}\right) \sqrt{C_1 \beta_T \gamma_{T^{(i)}}} \\
        &\leq 4 \left(\log_2 \log_2 T^{(i)} + 1\right) \left(\sqrt{T^{(i)}} + T^{(i)-1/4}\right) \sqrt{C_1 \beta_T \gamma_{T^{(i)}}}.
    \end{align}
    In the second line from the bottom, we used the fact that $\frac{N_{j}^{(i)}}{\sqrt{N_{j-1}^{(i)}}} \leq \sqrt{T^{(i)}} + T^{(i)-1/4}$ holds (for example, see the proof of Theorem~1 in \cite{li2022gaussian}).
    The last line follows from Proposition~1 in \cite{li2022gaussian}.
\end{proof}

\begin{proof}[Proof of \thmref{thm:regret}]
    Let $R_T^{(i)} \coloneqq \sum_{j=1}^{Q^{(i)}} \sum_{m=1}^{N_j^{(i)}} f_{j,m}^{(i)}(\bx_{j,m}^{(i)\ast}) - f_{j,m}^{(i)}(\bx_{j,m}^{(i)})$ be the regret incurred over $i$-th interval. 
    Then, we decompose the regret $R_T^{(i)}$ as
    \begin{align}
        \label{eq:thm:regret-first}
        R_T^{(i)} 
        =& \sum_{m=1}^{N_1^{(i)}} f_{1,m}^{(i)}\left(\bx_{1,m}^{(i)\ast}\right) - f_{1,m}^{(i)}\left(\bx_{1,m}^{(i)}\right)
        + \sum_{j=2}^{N_j^{(i)}} \sum_{m=1}^{N_j^{(i)}} f_{j,m}^{(i)} \left(\bx_{j,m}^{(i)\ast}\right) - \overline{f}_{j-1}^{(i)}\left(\tilde{\bm{x}}_{j-1}^{(i)\ast}\right) \\
        +& \sum_{j=2}^{N_j^{(i)}} \sum_{m=1}^{N_j^{(i)}} \overline{f}_{j-1}^{(i)}\left(\tilde{\bm{x}}_{j-1}^{(i)\ast}\right) - \overline{f}_{j-1}^{(i)}\left(\bm{x}_{j, m}^{(i)}\right)
        + \sum_{j=2}^{N_j^{(i)}} \sum_{m=1}^{N_j^{(i)}} \overline{f}_{j-1}^{(i)}\left(\bm{x}_{j, m}^{(i)}\right) - f_{j,m}^{(i)}\left(\bx_{j,m}^{(i)}\right).
    \end{align}
    By using the aforementioned lemmas, we can obtain the upper bound of each term as follows:
    \begin{itemize}
        \item Since $N_1^{(i)} \leq \sqrt{T^{(i)}} + 1$, the first term is bounded from above by $2B (\sqrt{T^{(i)}} + 1)$.
        \item From Lemmas~\ref{lem:fx_fox}--\ref{lem:of_oftx}, the 
        second term is bounded from above by $(6 + \log_2 \log_2 T^{(i)}) T^{(i)} V_T^{(i)}$.
        \item From \lemref{lem:pe_bound}, the third term is bounded from above by $4 \left(\log_2 \log_2 T^{(i)} + 1\right) \left( \sqrt{T^{(i)}} + T^{(i)-1/4}\right) \sqrt{C_1 \gamma_{T^{(i)}} \beta_T}$.
        \item From \lemref{lem:of_f}, the fourth term is bounded from above by $T^{(i)} V_T^{(i)}$.
    \end{itemize}
    Combining the above upper bounds with $T^{(i)} \leq H$ and $V_T = \sum_{i=1}^{\lceil T/H \rceil} V_T^{(i)}$, we have
    \begin{align}
        R_T 
        =& \sum_{i=1}^{\lceil T/H \rceil} R_T^{(i)} \\
        \begin{split}
            \label{eq:total_regret_upper}
            \leq& (7 + \log_2 \log_2 H) V_T H \\
            &+ \left\lceil \frac{T}{H} \right\rceil \sbr{2B \rbr{\sqrt{H}+ 1} + 4\rbr{\log_2 \log_2 H + 1} \rbr{\sqrt{H} + 1} \sqrt{C_1 \gamma_{H} \beta_T} }.
        \end{split}
    \end{align}
    The desired results are obtained by choosing the proper $H$ in the above inequality.
\paragraph{For SE kernel}
When $k = k_{\text{SE}}$, 
    $\gamma_H = \mathcal{O}\left(\ln^{d+1} H\right)$.
    Therefore, we have
    \begin{align}
        \label{eq:se_reg_fixed_H}
        R_T = \mathcal{O}\left( V_T H \log_2 \log_2 T + T \left( \log_2 \log_2 T\right) \left(\ln T\right)^{1/2}\sqrt{\frac{\ln^{d+1}H}{H}}\right).
    \end{align}
    Here,
    \begin{align}
        \label{eq:se_H_org}
        &H \geq T^{2/3} V_T^{-2/3} (\ln T)^{(d+2)/3} \\
        &\Leftrightarrow H^{3/2} \geq V_T^{-1} T (\ln T)^{1/2} (\ln T)^{(d+1)/2} \\
        &\Rightarrow H^{3/2} \geq V_T^{-1} T (\ln T)^{1/2} (\ln H)^{(d+1)/2} \\
        \label{eq:se_H_res}
        &\Rightarrow V_T H \log_2 \log_2 T \geq T \left( \log_2 \log_2 T\right) \left(\ln T\right)^{1/2}\sqrt{\frac{\ln^{d+1}H}{H}}.
    \end{align}
    Therefore, by setting 
    $H = \left\lceil T^{2/3} V_T^{-2/3} (\ln T)^{(d+2)/3} \right\rceil$,
    \begin{align}
        &V_T H \log_2 \log_2 T + T \left( \log_2 \log_2 T\right) \left(\ln T\right)^{1/2}\sqrt{\frac{\ln^{d+1}H}{H}} \\
        &\leq 2 V_T \left( T^{2/3} V_T^{-2/3} (\ln T)^{(d+2)/3} + 1 \right) \left( \log_2 \log_2 T \right) \\
        & = \tilde{O}\rbr{V_T^{\frac{1}{3}} T^{\frac{2}{3}}}.
        \label{eq:se_resulting_reg}
    \end{align}
\paragraph{For Mat\'ern kernel} When $k = k_{\text{Mat\'ern}}$, $\gamma_H = \mathcal{O}\left(H^{\frac{d}{2\nu + d}} \ln^{\frac{2\nu}{2\nu + d}} H\right)$.
    Therefore, 
    \begin{equation}
        \label{eq:mat_reg_fixed_H}
        R_T = \mathcal{O}\left( V_T H \log_2 \log_2 T + T \left( \log_2 \log_2 T\right) \left(\ln T\right)^{1/2}\sqrt{\frac{H^{\frac{d}{2\nu + d}} \ln^{\frac{2\nu}{2\nu + d}} H }{H}}\right).
    \end{equation}
    Now,
    \begin{align}
        \label{eq:mat_H_org}
        &H \geq V_T^{-\frac{2\nu+d}{3\nu + d}} T^{\frac{2\nu+d}{3\nu+d}} (\ln T)^{\frac{4\nu + d}{2(3\nu+d)}} \\
        &\Leftrightarrow H^{\frac{3\nu + d}{2\nu + d}} \geq V_T^{-1} T (\ln T)^{1/2} (\ln T)^{\frac{2\nu}{2(2\nu + d)}} \\
        &\Rightarrow H^{\frac{3\nu + d}{2\nu + d}} \geq V_T^{-1} T (\ln T)^{1/2} (\ln H)^{\frac{2\nu}{2(2\nu + d)}} \\
        \label{eq:mat_H_res}
        &\Leftrightarrow  V_T H \log_2 \log_2 T \geq T \left( \log_2 \log_2 T\right) \left(\ln T\right)^{1/2}\sqrt{\frac{H^{\frac{d}{2\nu + d}} \ln^{\frac{2\nu}{2\nu + d}} H }{H}}.
    \end{align}
    Therefore, by setting $H = \left\lceil V_T^{-\frac{2\nu+d}{3\nu + d}} T^{\frac{2\nu+d}{3\nu+d}} (\ln T)^{\frac{4\nu + d}{6\nu+2d}} \right\rceil$, we have
    \begin{align}
        &V_T H \log_2 \log_2 T + T \left( \log_2 \log_2 T\right) \left(\ln T\right)^{1/2}\sqrt{\frac{H^{\frac{d}{2\nu + d}} \ln^{\frac{2\nu}{2\nu + d}} H }{H}} \\
        &\leq 2 V_T H \log_2 \log_2 T \\
        &= \tilde{O}\left(T^{\frac{2\nu + d}{3\nu + d}}V_T^{\frac{\nu}{3\nu + d}} \right).
        \label{eq:mat_resulting_reg}
    \end{align}
\end{proof}

\section{Near-Optimal Version of OPKB for Mat\'ern RKHS}
\label{sec:mod_opkb}
We show that the OPKB algorithm with the restart-reset strategy 
can achieve near-optimal regret upper bound even in the Mat\'ern RKHS by properly selecting the restarting interval.
The formal statement is described in the following \thmref{thm:mod_opkb}.

\begin{theorem}[The modified version of OPKB algorithm with the restart-reset strategy.]
\label{thm:mod_opkb}
Assume that the underlying kernel is Mat\'ern kernel with smoothness parameter $\nu > 1/2$.
Furthermore, suppose that Assumptions~\ref{asmp:noise}--\ref{asmp:nons} and $V_T \geq T^{-\frac{\nu}{2\nu +d}} \ln^{\frac{\nu + d}{2\nu + d}} T$ hold. 
Then, if we set the restarting interval $H$ as $H = \lceil T^{\frac{2\nu + d}{3\nu + d}} V_T^{-\frac{2\nu + d}{3\nu + d}} \ln^{\frac{4\nu + d}{6\nu+2d}} T\rceil$, 
the OPKB algorithm (Algorithm 2 in \citet{hong2023optimization}) with the restart-reset strategy 
achieve $R_T = \tilde{\mO} (V_T^{\frac{\nu}{3\nu+d}} T^{\frac{2\nu + d}{3\nu + d}})$ with probability 
at least $1 - \delta$.
\end{theorem}

To show the above statement, it is enough to show that $R_T^{(i)} = \tilde{O}(V_T H + \sqrt{\gamma_H H})$, 
where $R_T^{(i)}$ denote the cumulative regret among $i$-th interval (also defined in the proof of \thmref{thm:regret}).
This is because that $\tilde{O}(V_T H + \sqrt{\gamma_H H})$ interval regret implies the same order of the cumulative regret as that of R-PERP (as shown in Eq.~\eqref{eq:total_regret_upper}); therefore, if we once obtain the 
$\tilde{O}(V_T H + \sqrt{\gamma_H H})$ interval regret, we can obtain the $\tilde{\mO} (V_T^{\frac{\nu}{3\nu+d}} T^{\frac{2\nu + d}{3\nu + d}})$ cumulative regret by following final part of the proof of \thmref{thm:regret}.

Note that the stationary base algorithm (without restart-reset strategy) with $\tilde{O}(V_T T + \sqrt{\gamma_T T})$ cumulative regret under non-stationary environment achieve $\tilde{O}(V_T H + \sqrt{\gamma_H H})$ interval regret when restart-reset strategy is applied.
Therefore, we show that $\tilde{O}(V_T T + \sqrt{\gamma_T T})$ cumulative regret of the standard OPKB algorithm
under Assumptions~\ref{asmp:noise}--\ref{asmp:nons}.

Hereafter, we use the same notation as those of \citet{hong2023optimization} for our proof unless otherwise specified.

\begin{proof}[Proof of \thmref{thm:mod_opkb}]
    As with the proof of Theorem~4.6 in \citet{hong2023optimization}, 
    we have the following inequality by leveraging Azuma-Hoeffding inequality\footnote{The setting of \citet{hong2023optimization} assumes the boundness assumption $f_t(\bx) \in [0, 1]$ of the reward function. 
    On the other hand, \asmpref{asmp:func} implies $f_t(\bx) \in [-B, B]$ in our setting. This difference in the magnitude of the reward function does not break the validity of the proof and appears as the difference of the constant factor.
    }:
    \begin{equation}
        R_{\mB(j)} \leq \tilde{O}(\sqrt{2^j E}) + \sum_{t \in \mB(j)} f_t(\bx_t^{\ast}) - \Ep_t[f_t(\bx_t)],
    \end{equation}
    where $\mB(j) \in [T]$ is the $j$-th batch index set of OPKB, and $R_{\mB(j)}$
    denote the cumulative regret incurred over $\mB(j)$.
    Furthermore, we decompose the second term into the following four terms:
    \begin{align}
        &f_t(\bx_t^{\ast}) - \Ep_t[f_t(\bx_t)] \\
        &= \underbrace{f_t(\bx_t^{\ast}) - f_t(\overline{\bx}_{\mC(j-1)}^{\ast})}_{A_1}  
        + \underbrace{f_t(\overline{\bx}_{\mC(j-1)}^{\ast}) - \overline{f}_{\mathcal{C}(j-1)}(\overline{\bx}_{\mC(j-1)}^{\ast})}_{A_2} \\
        &+ \underbrace{\Ep_t[\overline{f}_{\mathcal{C}(j-1)}(\overline{\bx}_{\mC(j-1)}^{\ast}) - \overline{f}_{\mathcal{C}(j-1)}(\bx_t)]}_{A_3} + \underbrace{\Ep_t[\overline{f}_{\mathcal{C}(j-1)}(\bx_t) - f_t(\bx_t)]}_{A_4},
    \end{align}
    where $\mC(j) \coloneqq \bigcup_{\tilde{j} \leq j} \mB(\tilde{j})$ is the cumulative time step set until the end of the batch $j$. Furthermore, $\overline{f}_{\mathcal{C}(j-1)}(\bx) \coloneqq \sum_{t \in \mC(j-1)} f_t(\bx) / |\mC(j-1)|$ and $\overline{\bx}_{\mC(j-1)}^{\ast}$ represent the average function over $\mC(j-1)$ and its maximizer, respectively. 
    As with the proof of stationary OPKB (Theorem~4.6 in \citet{hong2023optimization}), the upper bound of $A_3$ is obtained as $O(V_{\mC(j)} + \mu_j)$ using Lemma~4.5 in \citet{hong2023optimization}. 
    As for $A_1, A_2$, and $A_4$, we can easily confirm that the sum of those terms is $O(V_T)$ (e.g., 
    by relying on the same arguments as those of \lemref{lem:fx_fox}--\ref{lem:of_oftx}).
    By aggregating and arranging the above upper bounds, the cumulative regret $R_{\mB(j)}$ over $\mB(j)$ is bounded from above by $O(V_T |\mB(j)| + E\sqrt{2^j})$.
    Therefore, we have
    \begin{align}
        R_T 
        &= \sum_{j} R_{\mB(j)} \\
        &= O(V_T T) + \sum_{j} \tilde{O}(E\sqrt{2^j}) \\
        &= \tilde{O}(V_T T + E\sqrt{T/E}) \\
        &= \tilde{O}(V_T T + \sqrt{\gamma_T T \ln |\mX|}),
    \end{align}
    where the third line follows from Schwarz's inequality. 
    As described before the proof, the above upper bound implies the 
    $\tilde{O}(V_T H + \sqrt{\gamma_H H})$ interval regret when the OPKB algorithm is used as the base algorithm 
    of the restart-reset-based procedures. 
    Therefore, the proof is completed.
\end{proof}

\begin{table}[t]
    \centering
    \caption{The comparison of existing and our algorithms for regrets and computational costs under the setting where $V_T$ is unknown. For the regret upper bound described in the table, we assume that $V_T$ satisfies $V_T > c$, where $c > 0$ is any fixed constant.}
    \label{tab:comp_alg_unknown_vt}
    \begin{tabular}{c|c|c|c}
        Algorithm & Regret (SE) & Regret (Mat\'ern) &
        \begin{tabular}{c}
             Computational cost  \\
             at step $t \leq T$
        \end{tabular} \\ \hline
        R/SW-GP-UCB & $\tilde{\mO}(T^{\frac{3}{4}} V_T)$ & $\tilde{\mO}(T^{\frac{12\nu + 13d}{16\nu + 8d}} V_T) $ & $\mO(|\mX|t^2)$ \\
        WGP-UCB & $\tilde{\mO}(T^{\frac{3}{4}} V_T)$ & 
        $\tilde{\mO}(T^{\frac{12\nu + 13d}{16\nu + 8d}} V_T) $ & $\mO(|\mX|t^2)$ \\
        OPKB & $\tilde{\mO}(T^{\frac{2}{3}} V_T^{\frac{1}{3}})$ 
        & $\tilde{\mO}(T^{\frac{4\nu + 3d}{6\nu + 3d}} V_T^{\frac{1}{3}})$ & $\mO(M |\mX|^3)$ \\
        R-PREP (Ours) & $\tilde{\mO}(T^{\frac{2}{3}} V_T)$ & 
$\tilde{\mO}(T^{\frac{2\nu + d}{3\nu + d}} V_T)$ & $\mO(|\mX|t^2)$ 
    \end{tabular}
\end{table}

\section{Regret Upper Bound for Unknown $V_T$}
\label{sec:unknown_vt}

In this section, we derive the regret upper bound for R-PERP in a setting where the total drift upper bound \( V_T \) is unknown. 
Importantly, in the proof in \thmref{thm:regret}, the prior information of $V_T$ is used only for tightening the regret upper bound.
That is, Lemmas~\ref{lem:fx_fox}-\ref{lem:pe_bound} hold for any $H \in [T] \setminus \{1 \}$.
Therefore, the derivation of Eqs.~\eqref{eq:thm:regret-first}-\eqref{eq:total_regret_upper} can be obtained even if $H$ is determined independently of $V_T$.
Hence, the following regret upper bound can be obtained by choosing \( H \) independently of \( V_T \) and slightly modifying the derivation of Eqs.~\eqref{eq:se_reg_fixed_H}-\eqref{eq:se_resulting_reg} and Eqs.~\eqref{eq:mat_reg_fixed_H}-\eqref{eq:mat_resulting_reg} for SE and Mat\'ern kernels, respectively:
%

\begin{itemize}
    \item \textbf{SE Kernel}: By setting \( H = \lceil T^{2/3} (\ln T)^{(d+2)/3}\rceil \), it can be shown in a similar manner to Eq.~\eqref{eq:se_H_org}--\eqref{eq:se_H_res} that 
    \begin{align}
        H \log_2 \log_2 T \geq T \left( \log_2 \log_2 T\right) \left(\ln T\right)^{1/2}\sqrt{\frac{\ln^{d+1}H}{H}}.
    \end{align}
    Therefore, $R_T = \tilde{O}(T^{2/3} (V_T+1))$ is satisfied by R-PERP.
    \item \textbf{Matérn Kernel}: By setting \( H = \lceil T^{\frac{2\nu + d}{3\nu + d}} (\ln T)^{\frac{4\nu + d}{6\nu + 2d}}\rceil \), we have
    \[
    H \log_2 \log_2 T \geq T \left( \log_2 \log_2 T\right) \left(\ln T\right)^{1/2}\sqrt{\frac{H^{\frac{d}{2\nu + d}} \ln^{\frac{2\nu}{2\nu + d}} H }{H}}
    \]
    in a similar manner to Eq.~\eqref{eq:mat_H_org}--\eqref{eq:mat_H_res}; thus, combining with Eq.~\eqref{eq:mat_reg_fixed_H}, it can be shown that R-PERP satisfies the regret upper bound $R_T = \tilde{O}(T^{\frac{2\nu + d}{3\nu + d}} (V_T+1))$.
\end{itemize}

\tabref{tab:comp_alg_unknown_vt} lists the regret and computational complexity of each method in the setting where \( V_T \) is unknown. In the setting where \( V_T \) is unknown, OPKB shows better theoretical performance than R-PERP in terms of the dependence of \( V_T \) by utilizing an appropriate adaptive reset scheduling. Extending R-PERP based on the ideas of adaptive reset scheduling of OPKB is an important future research direction. On the other hand, the computational complexity issue with OPKB exists regardless of whether \( V_T \) is known or unknown. Therefore, R-PERP shows the best regret guarantees among the algorithms applicable in scenarios where \( \mathcal{X} \) is huge, such as R/SW-GP-UCB and W-GP-UCB.

\end{document}